\documentclass{article}

\usepackage{microtype}
\usepackage{graphicx}
\usepackage{subcaption}
\usepackage{booktabs} 

\usepackage{hyperref}
\usepackage{xurl}

\usepackage[preprint]{icml2026}

\usepackage{amsmath}
\usepackage{amssymb}
\usepackage{mathtools}
\usepackage{amsthm}

\usepackage{url}
\usepackage{amsfonts}
\usepackage{nicefrac}
\usepackage{xcolor}
\usepackage{multirow}
\usepackage{paralist}
\usepackage{xspace}

\newcommand{\method}{GSUO\xspace}
\newcommand{\signalopt}{SGO\xspace}
\newcommand{\compopt}{CO\xspace}

\usepackage[capitalize,noabbrev]{cleveref}

\theoremstyle{plain}
\newtheorem{theorem}{Theorem}[section]

\theoremstyle{definition}

\theoremstyle{remark}

\usepackage[textsize=tiny]{todonotes}

\icmltitlerunning{Signal-Guided Optimization for Machine Unlearning}

\begin{document}

\twocolumn[
  \icmltitle{Signal-Guided Optimization for Machine Unlearning}



  \icmlsetsymbol{equal}{*}

  \begin{icmlauthorlist}
    \icmlauthor{Xujia Li}{ustc}
    \icmlauthor{Dan Li}{sysu}
    \icmlauthor{Jian Lou}{sysu}
    \icmlauthor{Wenjie Feng$^{*}$}{ustc}
\end{icmlauthorlist}

\icmlaffiliation{ustc}{University of Science and Technology of China}

\icmlaffiliation{sysu}{Sun Yat-sen University}

\icmlcorrespondingauthor{Wenjie Feng}{fengwenjie@ustc.edu.cn}

  \vskip 0.3in
]



\printAffiliationsAndNotice{}  

\begin{abstract}
  Current machine unlearning methods predominantly rely on global, coarse-grained intervention strategies. 
  They lack precise pilot signals to guide the unlearning process and 
  fail to provide differentiable guidance across different unlearning tasks. 
  Due to the varying memorization strengths of samples during original training, 
  such a uniform strategy leads to two problems: 
  some samples are over-unlearned, which harms model utility; 
  while others are under-unlearned, leaving residual information that can be exploited by privacy attacks. 
  In this paper, we propose \method, a guidance-signal-aware unlearning optimization framework that 
  designs task-specific fine-grained guidance signals to steer the unlearning process
  and is applicable to both random-subset and class-wise forgetting tasks.
  Extensive experiments demonstrate that \method outperforms 14 baselines in terms of both unlearning effectiveness and generalization, 
  while achieving high efficiency and significant speedups, validating its effectiveness for reliable machine unlearning.
\end{abstract}

\section{Introduction}
With growing awareness of privacy preservation and the legal guarantee of the ``right to be forgotten'' under regulations such as GDPR~\citep{VB17gdpr},
deep models are expected to possess the capability to remove the influence of specific samples from training data.
Moreover, outdated information or maliciously poisoned samples in the training data further reinforce the need for models to erase such data.


To alleviate the prohibitive overhead of retraining, which involves training from scratch on the retain data and is commonly regarded as the gold standard~\citep{DBLP:journals/tist/NguyenHRNLYN25, DBLP:journals/csur/XuZZZY24}, 
various approximate unlearning methods have consequently been proposed~\citep{DBLP:conf/aaai/ChundawatTMK23,DBLP:conf/nips/KurmanjiTHT23,DBLP:journals/tnn/TarunCMK24,DBLP:conf/cvpr/ChenGL0W23,DBLP:conf/kdd/ChenG0LA025,DBLP:journals/corr/abs-2601-18650,DBLP:conf/iclr/FanLZ0W024}.
For instance, SCRUB~\citep{DBLP:conf/nips/KurmanjiTHT23} employs a teacher-student framework with contrastive optimization, and SalUn~\citep{DBLP:conf/iclr/FanLZ0W024} identifies salient weights based on forgetting loss gradients, followed by random-label fine-tuning. 
For class-wise forgetting, Boundary Shrink alters the labels of forget samples via adversarial perturbations, 
and Boundary Expanding introduces a ``shadow class'' neuron to disperse such samples~\citep{DBLP:conf/kdd/ChenG0LA025}.

However, approximate unlearning methods typically either perform unified optimization for all forget samples without differentiation~\citep{DBLP:conf/aaai/ChundawatTMK23, DBLP:conf/nips/KurmanjiTHT23} or solely modify labels directly~\citep{DBLP:conf/kdd/ChenG0LA025, DBLP:conf/iclr/FanLZ0W024},
resulting in the neglect of varying memorization strengths of different samples during training, 
their intervention strategies therefore are generally global and coarse-grained, 
lacking precise signals to finely guide the unlearning process. 
As a result, it leads to \emph{either over-unlearning that degrades model generalization 
or incomplete unlearning that leaves residual information vulnerable to privacy attacks}. 

To this end, 
we propose the \method framework
to exert unlearning relying on targeted guidance signals 
which are designed for diverse samples based on their context and properties across different forgetting tasks,
aiming to reduce or avoid over-forgetting and under-forgetting.
Specifically, for random-subset forgetting, where forget samples are uniformly scattered in feature space and share similarities with other retain data, 
\method outlines the guidance signals as diverse target distributions to align according to the categories (i.e. \textit{Normal} and \textit{Boundary}) forget samples belong to;
for class-wise forgetting, where forget samples form tight clusters in feature space, \method employs intra-class dispersion loss and alignment loss as the guidance signals to prevent performance degradation resulting from disruption of the class boundaries learned by the model.
Thus, by explicitly specifying the unlearned state of the forget samples, \method can perform tailored optimization with a well-defined objective.



Comprehensive experimental results demonstrate that \method consistently achieves the best performance across a wide range of metrics, including fidelity, generalization ability, and forgetting efficacy, 
while it maintains high efficiency and delivers substantial speed-ups, outperforming all 14 state-of-the-art baselines. 
On the random-subset forgetting task, 
\method achieves the smallest accuracy gap between test and forget set, reaching $0.17\%$ for ResNet-18, 
along with the highest test accuracy of $82.08\%$.  
It also delivers a substantial speed-up of up to $31\times$ and the lowest AUC score of $0.749$ under strong membership inference attacks (MIA). 
For class-wise forgetting, \method demonstrates comparable performance, achieving complete unlearning while maintaining the highest training and test accuracies.


\section{Related Work}
\label{sec:related}
Machine Unlearning was first introduced by \citep{DBLP:conf/sp/CaoY15} to eliminate the influence of training samples 
on a learned model efficiently and completely by sharding data and building multiple models, 
enabling the exact unlearning of specific data partitions. 
For approximate unlearning, numerous methods have been developed to efficiently erase the influence of data through post-training procedures. 
Methods leveraging second-order information or influence functions~\citep{DBLP:conf/aistats/IzzoSCZ21, DBLP:conf/nips/SekhariAKS21, DBLP:conf/icml/KohL17, DBLP:conf/icml/GuoGHM20} approximately estimate the contribution of samples to the loss function to revoke their influence. 
However, most of these approaches are limited to convex models or incur high computational costs when applied to deep networks. 
For deep neural networks, \citep{DBLP:conf/eccv/GolatkarAS20, DBLP:conf/cvpr/GolatkarAS20} proposed more practical approaches based on information projection and quadratic penalty, respectively. 
\citep{DBLP:conf/iclr/FanLZ0W024} proposed SalUn , which computes a weight saliency mask to selectively update only the parameters most relevant to the forget data, achieving an excellent balance between preserving performance on the retain set and unlearning effectiveness. 
\citep{DBLP:conf/aaai/ChundawatTMK23,DBLP:conf/nips/KurmanjiTHT23} employ a teacher-student framework for unlearning, with the former using a dual-teacher mechanism and the latter providing directional guidance from a single teacher. In such frameworks, the soft labels generated by the teacher model guide the student model to learn a new output distribution. 
\citep{DBLP:journals/pami/HeLCHH25} proposed \emph{Natural Unlearning} to generate new samples by mixing each forget sample with related samples from the retain set, and then fine-tuning the original model on these mixed samples. 

Another line of work focuses on 
class forgetting, which requires the model to completely remove its knowledge of an entire concept or class. 
\citep{DBLP:journals/tnn/TarunCMK24} propose UNSIR, which learns ``error-maximizing noise'' for the target class and combines a ``damage-repair'' two-step weight update to achieve thorough forgetting of single or multiple classes in a single training pass. 
\citep{DBLP:conf/cvpr/ChenGL0W23,DBLP:conf/kdd/ChenG0LA025} shift the focus from the parameter space to the decision space, 
Boundary Shrink assigns incorrect labels to forget samples via adversarial perturbations and fine-tunes the model, forcing the decision boundary to shrink; while Boundary Expanding introduces a new neuron for ``shadow class'' to disperse forget samples into other categories, actively expanding the decision boundary. \citep{DBLP:conf/kdd/ChenG0LA025} and \cite{DBLP:journals/corr/abs-2601-18650} focus on achieving class-wise forgetting on long-tailed distribution datasets.

A detailed discussion of related work is provided in Appendix~\ref{app:B}.

\section{Problem Formulation}
\label{sec:problem}

\paragraph{Notation} 
Given a training dataset 
$\mathcal{D} = \left\{(x_i, y_i)\right\}_{i=1}^N \subseteq \mathcal{X} \times \mathcal{Y}$ consisting of $N$ samples
where $x_i \in \mathcal{X}$ and $y_i \in \mathcal{Y}$ are independently and identically distributed drawn from a joint distribution $\mathcal{P}$ over $\mathcal{X} \times \mathcal{Y}$; $\mathcal{X} \subseteq \mathbb{R}^F$ with dimension as $F$ and label space $\mathcal{Y} = \{1, \ldots, C\}$ with $C$ classes.
Let $f(\cdot; \Theta)$ denote a deep neural network model with trainable parameters $\Theta$. The model $f(\cdot; \Theta_o)$ trained over $\mathcal{D}$ producing $\Theta_o$ is referred to as the \textit{original model}, which consists of a feature extractor $\Phi(\cdot; \psi)$ and a classifier $h(\cdot; \omega)$, i.e., $f(x; \Theta_o) = h(\Phi(x; \psi); \omega)$; therefore, $\Theta_o = \{\psi, \omega\}$.

\paragraph{Machine Unlearning} Let $\mathcal{D}_f \subset \mathcal{D}$ be a subset of the training dataset 
composed of $N_f$ samples as the \textit{forget set}, and its complement $\mathcal{D}_r$ be the retain set, that is,
\(\mathcal{D} = \mathcal{D}_f \cup \mathcal{D}_r\) with \(\mathcal{D}_f \cap \mathcal{D}_r = \emptyset\). 
The original model $f(\cdot; \Theta_o)$ is trained on \(\mathcal{D}\) by minimizing the cross-entropy loss \(\mathcal{L}(\hat{y}, y) = -\sum_{k=1}^C y_k \log(\hat{y}_k)\). 
The retrained model, denoted as $f(\cdot; \Theta_*)$, is obtained by training on the retain set \(\mathcal{D}_r\) alone. 
The goal of machine unlearning is to eliminate the influence of \(\mathcal{D}_f\) from $f(\cdot; \Theta_o)$ via some unlearning algorithm \(U\), producing updated weights \(\Theta_u\) such that the unlearned model $f(\cdot; \Theta_u)$ behaves as if it had never seen \(\mathcal{D}_f\). 
In other words, the performance of $f(\cdot; \Theta_u)$ on \(\mathcal{D}_f\) should approximate its performance on a held-out test set $\mathcal{D}_t$, i.e.,
$f(\mathcal{D}_f; \Theta_u) \approx f(\mathcal{D}_t; \Theta_o)$.

\paragraph{Random-subset and class-wise forgetting}
According to the composition of the forget set \(\mathcal{D}_f\), the unlearning tasks can be categorized into the following two categories. 
\begin{compactitem}
    \item \textit{Random-subset forgetting}: 
    $\mathcal{D}_f$ is an arbitrary subset of $\mathcal{D}$, 
    $\mathcal{D}_f \subset \mathcal{D}$ with $|\mathcal{D}_f| = N_f \ll N$, and the samples to be forgotten may come from multiple classes of $\mathcal{Y}$. 
    \item \textit{Class-wise forgetting}: 
    $\mathcal{D}_f$ contains all training samples of a subset of target classes $\mathcal{Y}_f \subset \mathcal{Y}$: $ \mathcal{D}_f = \{(x_i, y_i) \in \mathcal{D} \mid y_i = c, \forall c \in \mathcal{Y}_f\}$. $|\mathcal{Y}_f| = 1$ reduces to single-class forgetting.
\end{compactitem}

\section{Proposed method: \method}
\label{SGA-TO}

\begin{figure*}[t]
    \centering
    \includegraphics[width=0.9\linewidth]{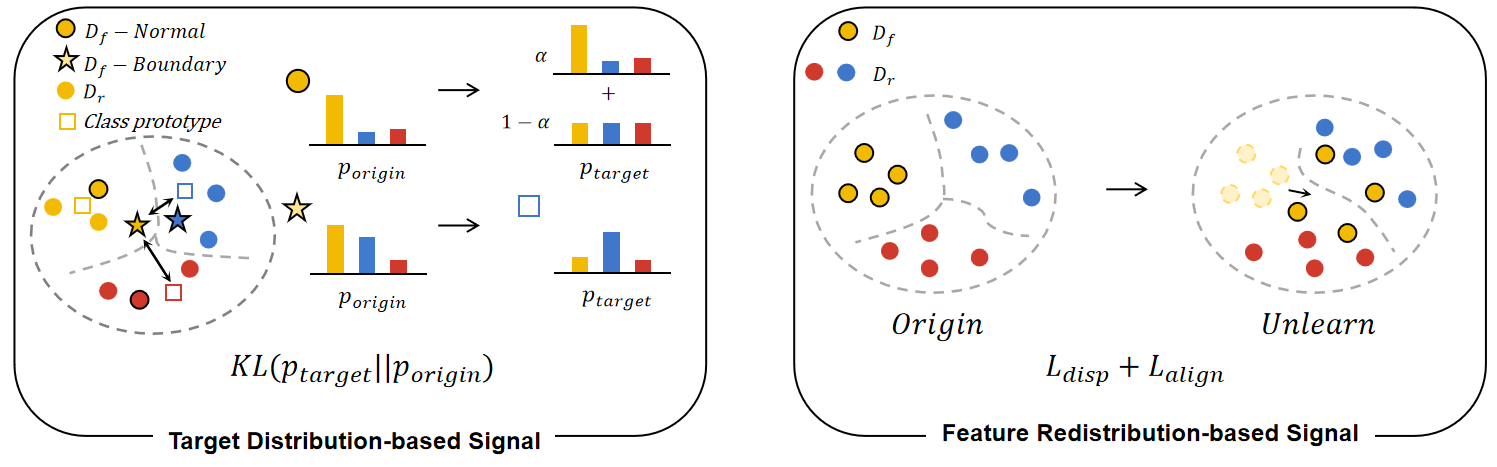}
    \caption{The guidance-signal design in the \signalopt step of \method. 
      \textbf{L}: Target Distribution-based Signal for random-subset forgetting, 
      \textbf{R}: Feature Redistribution-based Signal for class-wise forgetting.}
    \label{fig:framework}
\end{figure*}

In this section, we propose the Guidance-Signal-aware Unlearning Optimization (\textbf{\method}), 
a unified framework for performing sample unlearning based on explicitly designed guidance signals,
which can be applied to both random-subset and class-wise forgetting tasks.
Specifically, \method consists of the following two successive optimization steps:
\begin{compactitem}
    \item \textbf{S1.} \textit{Signal-Guided Optimization (\signalopt)}: leveraging specifically designed guidance signals for diverse samples to steer the unlearning process.
    \item \textbf{S2.} \textit{Compensatory Optimization (\compopt)}: employing a general cross-entropy loss over the retain set and/or the forget set 
    to enhance model utility via gradient descent/ascent (GD/GA). 
\end{compactitem}

For the guidance signal in \signalopt, based on the sample behavior in the ideal unlearned model,
we design \emph{Target Distribution-based Signal} for random-subset forgetting task
and \emph{Feature Redistribution-based Signal} for class-wise forgetting task, 
which will be elaborated in detail in the following subsections.
\figureautorefname{~\ref{fig:framework}} illustrates the design and principle of these guidance signals.


\subsection{Random-subset Forgetting}
Under \method paradiagm, guidance signals leveraged by \signalopt of \textbf{S1} 
are specified as diverse target distributions associated with different forget sample categories;
while \compopt in \textbf{S2} performs fine-tuning with GD on $\mathcal{D}_r$.
Note that forget samples are uniformly scattered in feature space and share similarities with other retain data in this case.

\subsubsection{\signalopt under Target Distribution-based Signal}
The core idea of the target distribution signal $p_{\mathrm{target}}$ is to identify predictive distributions
that forget samples are expected to follow based on their location in the feature space of $f(\cdot; \Theta_o)$, 
and then align the model's prediction for the forget sample with the corresponding target distribution to achieve unlearning. 
Due to the varying strength of memorization of different samples, adopting a coarse-grained update strategy that treats all forgotten samples equally leads to a dilemma: some samples are under-unlearned, while others are over-unlearned (please refer to the detailed theoretical analysis in Appendix~\ref{app:A}).
Hence, we categorize forget samples into \textit{Normal} and \textit{Boundary} based on the model's confidence, 
which intuitively is inversely proportional to the distance between the sample and the decision hyperplane,
and design corresponding guidance signals considering their behavior under ideal unlearning.
Specifically, the overall schema is as follows.

\paragraph{Normal} They are high-confidence samples in the original model
and remembered by the model as a part of the training set, or perhaps they are overconfident. 
Consequently, these samples should be expected to undergo perturbation or a decrease in predicted confidence after unlearning,
while remaining similar to their neighbors. 
As a result, they should \emph{maintain isotropic or uniform randomness among neighbors} 
when no complete prior knowledge about their exact state after unlearning.
Therefore, to facilitate a smooth unlearn, 
we construct the corresponding target distribution for a forget sample $x_n$ as a linear combination of the original model prediction and a uniform distribution, i.e.,
\begin{equation}
p_{\mathrm{target}} = \alpha \cdot \mathrm{softmax}\bigl(f(x_{\mathrm{n}}; \Theta_o))\bigr) + (1-\alpha) \cdot \mathbf{1} / C,
\end{equation}
where $\mathbf{1}$ is a $C$-size all-one vector, which becomes a uniform distribution after normalization by $\frac{1}{C}$, and $\alpha \in [0,1]$ controls the forgetting strength. 
Such a guidance signal $p_{\mathrm{target}}$ steers the model to gradually reduce the confidence in the original predicted class for forget samples.

\paragraph{Boundary} They are relatively low-confidence samples in the original model 
and almost are close to the decision hyperplane.
Consequently, they will \emph{tend to cross decision boundaries after unlearning, leading to misclassification}. 
Accordingly, the corresponding guidance signal (target distribution) is designed as the predictive distribution 
of the most similar retained class in the feature space.

Here, we first define the target signal for each class $c \in \mathcal{Y}$, 
which takes the prediction distribution $P_c$ of the corresponding prototype $\mu_c$ as a proxy.
Specifically, we sample a subset $\mathcal{D}^c_P = \{(x, y) \in \mathcal{D}_r | y = c \}$ of size $R \ll |\mathcal{D}_r|$
and extract their feature by the extractor $\Phi(\cdot; \psi)$, the prototype is formulated as
$\mu_c=\frac{1}{R}\sum_{x \sim \mathcal{D}^c_P}\Phi(x; \psi)$, 
i.e., the center of class $c$ in feature space.

Therefore, the target distribution signal for any boundary sample $x_{\mathrm{b}}$ is $p_{\mathrm{target}} = P_{c^*}$ where
\begin{equation}
    c^* = \arg\min_{c \neq y_{\mathrm{true}}} \| \Phi(x_{\mathrm{b}}; \psi)) - \mu_c \|^2.
\end{equation}
Thus, it guides boundary samples to ``blend into'' semantically similar retain classes, 
achieving unlearning with minimal impact on the retain classes.



Relying on the above $p_{\mathrm{target}}$, we align the model's predictions on the forget samples 
with the target distribution signal by minimizing the KL divergence, i.e.,
\(\mathrm{KL}(p_{\mathrm{target}} \parallel p_{\mathrm{model}})\).
Normal and Boundary samples equally contribute to the total KL loss.

\subsubsection{\compopt for enhancement}
To address the potential degradation of the model's discriminative ability after \textbf{S1},
we can fine-tune on $\mathcal{D}_r$ (optionally only its subset) to enhance for a few epochs with a standard cross-entropy loss. 

\subsection{Class-wise Forgetting}
Under \method paradigm, the guidance signal in \textbf{S1} is based on feature redistribution. 
In \textbf{S2}, we apply cross-entropy loss on $\mathcal{D}_r$ to reinforce correct classification and 
negative cross-entropy loss on $\mathcal{D}_f$ to suppress prediction ability for forget classes, 
both with a layer-wise learning rate strategy. 
For this task, forget samples form tight clusters in feature space with clear inter-class decision boundaries.

\subsubsection{\signalopt under Feature Redistribution-based Signal}
The core idea of feature distribution signals is to reshape the distribution of forget samples 
in feature space according to their characteristics for thorough class-wise forgetting. 
As the empirical observation from retraining shows, 
\emph{the forget samples become loosely distributed and their intra-class compactness is disrupted after unlearning},
which serves as the prior knowledge for designing efficient unlearning methods.
Therefore, we construct two types of guidance signals: \textit{a dispersion loss} and \textit{an alignment loss},
to steer unlearning objective by working together. 

\paragraph{Intra-class dispersion loss:} 
$\mathcal{D}_f$ contains all training samples of a subset of classes $\mathcal{Y}_f \subseteq \mathcal{Y}$. 
We define the dispersion loss $L_{\mathrm{disp}}$ as the negative logarithm of the average pairwise distance between the normalized features from the same forget class, that is,
\begin{equation}
\begin{aligned}
L_{\mathrm{disp}}
=&
-\frac{1}{|\mathcal{Y}_f|}
\sum_{\mathcal{D}_f^c}
\log \Bigg(
\frac{1}{N_c(N_c-1)}
\sum_{\substack{x_i,x_j\in\mathcal D_f^c\\ i\neq j}}
\\
&
\left(
1-
\cos\left(
\frac{\Phi(x_i;\psi)}{\|\Phi(x_i;\psi)\|},
\frac{\Phi(x_j;\psi)}{\|\Phi(x_j;\psi)\|}
\right)
\right)
\Bigg),
\end{aligned}
\end{equation}

where $N_c = |\mathcal{D}_f^c|$ 
and $\cos(\cdot, \cdot)$ denotes cosine similarity. 
By maximizing the pairwise distances among intra-class features, 
$L_{\mathrm{disp}}$ forces the samples from forget class to scatter across the feature space, 
thereby disrupting their original clustering structure.

\paragraph{Alignment loss:} 

In addition to intra-class dispersion after unlearning, the forget samples as a whole 
tend to shift toward the global center of the retained classes.
Ignoring such an overall offset would weaken the forgetting performance where  
the scattered features still reside near their original region, which allows the classifier to make correct predictions through other means (e.g., nearest neighbor classification). 
Therefore, we introduce an alignment loss that drives the feature center of each forget class toward the retained region.

Here, we define the global center of the retained classes as 
$\bar{\mu}_r = \frac{1}{|\mathcal{D}_r|} \sum_{(x,y) \in \mathcal{D}_r} \Phi(x; \psi)$
and the normalized feature center for each forget class as 
$\mu_f^{c} = \frac{1}{N_c} \sum_{x \in \mathcal{D}_f^c} \frac{\Phi(x; \psi)}{\|\Phi(x; \psi)\|}$. 
The alignment loss is
\begin{equation}
L_{\mathrm{align}} = \frac{1}{|\mathcal{Y}_f|} \sum_{c \in \mathcal{Y}_f} \left( 1 - \cos\left(\mu_f^{c}, \bar{\mu}_r\right) \right),    
\end{equation}
where $\mathcal{Y}_r = \mathcal{Y} \setminus \mathcal{Y}_f$ is the set of retained classes.
Hence, the weighted sum of the dispersion loss and the alignment loss forms the guidance signal for unlearning.

\subsubsection{\compopt for enhancement}
In \textbf{S2}, we jointly optimize the feature extractor and classifier using the cross-entropy loss, reinforcing correct classification of retained samples while suppressing accurate predictions for forget samples.
We adopt a layer-wise learning rate strategy: a smaller learning rate is applied to the feature extractor to maintain the achieved feature reshaping effect, 
while a larger learning rate is applied to the classifier to adapt it to the new feature distribution and enhance its discriminative ability for the retained classes.

\section{Experiements}
\subsection{Experimental Settings}
\paragraph{Datasets} 
We conduct experiments on CIFAR-10 and CIFAR-100 \citep{krizhevsky2009learning}, Lacuna-10 (randomly select 10 classes with at least 500 samples from VGG-Faces \citep{DBLP:conf/sp/CaoY15}), and Tiny-ImageNet \citep{le2015tiny}, similar to previous work, to measure the unlearning performance for classification tasks. 

For \textit{random-subset forgetting}, we choose to forget $\approx 10\%$ of whole training data,i.e., $N_f = |\mathcal{D}_v| = 5,000$ for CIFAR-10 and CIFAR-100 and $N_f = 10k$ for Tine-ImageNet.
For \textit{class-wise forgetting}, we report the results about Lacuna-10 for forgetting a randomly selected class.

\paragraph{Models}
For Lacuna-10, we fine-tuned a pre-trained VGG-11 model using Adam with a fixed $lr =1E-3$ and batch size of $128$. 
For CIFAR-10, we fine-tuned a pre-trained ResNet18 model using Adam with a fixed $lr =1E-3$ and batch size of $256$. For CIFAR-100 and Tiny-ImageNet, we finetuned pre-trained VGG-16~\citep{DBLP:journals/corr/abs-2012-11849}, ResNet-50~\citep{DBLP:conf/cvpr/HeZRS16}, and ViT~\citep{DBLP:conf/iclr/DosovitskiyB0WZ21} models using Adam with a fixed $lr = 1E-4$ and batch-size of $256$. 
All methods are implemented in Python 3.11 with PyTorch. 
All experiments are performed on NVIDIA Tesla-A100 (80G) with Intel Xeon processors.

\paragraph{Baselines}
We covered \textbf{14 state-of-the-art baselines}: 
\textbf{Original}: the model trained on $\mathcal{D}$.
\textbf{Retrain}:retraining the model from scratch without $\mathcal{D}_f$(included as a reference but inviable in practice).
\textbf{Fine-tuning (FT)}, \textbf{Gradient Ascent (GA)}~\citep{DBLP:conf/cvpr/GolatkarAS20}, 
\textbf{Random Labels (RL)}~\citep{DBLP:conf/aaai/GravesNG21}, 
\textbf{Influence Unlearning (IU)}~\citep{DBLP:conf/icml/KohL17,DBLP:conf/aistats/IzzoSCZ21}, 
\textbf{EU-k}~\citep{DBLP:journals/corr/abs-2201-06640}, 
\textbf{$\ell_1$-sparse}~\citep{DBLP:conf/nips/JiaLRYLLSL23}, 
\textbf{SalUn}~\citep{DBLP:conf/iclr/FanLZ0W024}, 
\textbf{Bad-teaching (Bad-T)}~\citep{DBLP:conf/aaai/ChundawatTMK23} and 
\textbf{SCRUB}~\citep{DBLP:conf/nips/KurmanjiTHT23}. 
Besides, for the class-wise forgetting, we include specific methods: 
\textbf{Boundary Shrink (B-Shrink) \& Expanding (B-Expand)}~\citep{DBLP:conf/cvpr/ChenGL0W23} and 
\textbf{UNSIR}~\citep{DBLP:journals/tnn/TarunCMK24}.

Considering the comprehensive forgetting efficacy, model performance, and privacy protection, we follow the evaluation metrics below:
\begin{compactitem}
    \item{\textit{Accuracy:}} The accuracies on different datasets, including the retain training set ($\mathcal{D}_r$), retain test set ($\mathcal{D}_{rt}$), forget training set ($\mathcal{D}_f$), forget test set ($\mathcal{D}_{ft}$) and the full test set ($\mathcal{D}_t$);
    \item{\textit{Difference ($|\mathrm{Diff}|$):}} 
        The absolute accuracy difference between forget set and test set, i.e., $|\mathrm{Diff}| = |Acc(\mathcal{D}_f) - Acc(\mathcal{D}_t)|$, which measures whether the model still retains memory of the forgotten data after unlearning;
    \item{\textit{Acc-Index:}}  An integrated accuracy index for the random-subset unlearning task as 
            $F_{acc}(\mathcal{D}_t, \mathcal{D}_f)=Acc(\mathcal{D}_t) - |~\mathrm{Diff}~|$;
    \item{\textit{Time:}} The running time for performing mahcine unlearning method;
    \item{\textit{Speedup:}} The ratio of the time required for retraining on $\mathcal{D}_r$ to the unlearning method.
    \item{\textit{Attack Success Rate (ASR):}} The attack success rate of MIA, 
        measuring the model's effectiveness in privacy protection 
        with the ideal value being close to a random guess.
\end{compactitem}

\textit{Acc-Index} $F_{acc} \in [0, 1]$ accounts for both the generalization capacity and unlearning efficacy of the unlearned model, i.e., better accuracy on $\mathcal{D}_t$ while the accuracy of $\mathcal{D}_f$ should be as consistent with it as possible.
Detailed experimental settings and more results are deferred to Appendix~\ref{app:C} and~\ref{app:D}. 
The code is available at https://anonymous.4open.science/r/GSUO.

\subsection{Random-subset Forgetting}

\begin{table*}[t]
    \centering
    \caption{Performance comparison among baselines and \method 
    for the random-subset forgetting task on CIFAR-10 for models with different architectures.}
    \label{tab:compare_random_unlearn}
    
    \resizebox{\textwidth}{!}{
    \begin{tabular}{l|l|rr|rrrrrrrrr|r} \toprule
    \multicolumn{1}{c|}{\textbf{\begin{tabular}[c]{@{}c@{}}Dataset (model)\end{tabular}}} & \multicolumn{1}{c|}{\textbf{Metric}} & \multicolumn{1}{c}{\textbf{Original}} & \multicolumn{1}{c|}{\textbf{Retrain}} & \multicolumn{1}{c}{\textbf{FT}} & \multicolumn{1}{c}{\textbf{GA}} & \multicolumn{1}{c}{\textbf{RL}} & \multicolumn{1}{c}{\textbf{EU-k}} & \multicolumn{1}{c}{\textbf{IU}} & \multicolumn{1}{c}{\textbf{$\ell_1$-sparse}} & \multicolumn{1}{c}{\textbf{SalUn}} & \multicolumn{1}{c}{\textbf{Bad-T}} & \multicolumn{1}{c}{\textbf{SCRUB}} & \multicolumn{1}{|c}{\textbf{\method}} \\ \midrule
    \multirow{5}{*}{\begin{tabular}[c]{@{}c@{}}CIFAR-10\\ (ResNet-18)\end{tabular}} 
        & $Acc(\mathcal{D}_r)$ \, $\uparrow$ & 98.29  & 98.61  & 98.66 & 97.50 & 93.51 & \underline{98.90}  & 85.03  & 96.20  & 98.51 & 97.50  & 93.93  & \textbf{99.5} \\
        & $Acc(\mathcal{D}_t)$ \, $\uparrow$ & 80.80  & 80.61 & 81.18 & 79.54 & 76.19 & \underline{81.55} & 71.20  & 78.70   & 80.51   & 79.07   & 78.22   & \textbf{82.08} \\
        & $Acc(\mathcal{D}_f)$ \, & 98.34 & 81.24  & 96.95  & 96.32 & 93.64   & 98.72   & 85.44   & 86.52  & 92.26   & 73.64   & 89.48   & 82.25 \\
        & $|~\mathrm{Diff}~|$ \, $\downarrow$  & 17.54  & 0.63   & 15.77  & 16.78  & 17.45  & 17.17   & 14.24  & 7.82  & 11.75  & \underline{5.43}   & 11.26    & \textbf{0.17} \\
        & Time (s) \, $\downarrow$ & \textendash & 230.10  & 16.61   & \textbf{1.08}   & \underline{1.65}  & 85.28  & 41.07   & 65.98  & 86.95  & 12.34  & 11.93   & 7.06 \\ \midrule
   \multirow{5}{*}{\begin{tabular}[c]{@{}c@{}}CIFAR-100\\ (VGG-16)\end{tabular}} 
        & $Acc(\mathcal{D}_r)$ \, $\uparrow$ & 99.53 & 99.53 & 99.59 & 99.02 & 98.41 & \underline{99.88} & 99.53 & 99.85 & 99.53 & 99.41 & \textbf{99.95} & 87.02\\
        & $Acc(\mathcal{D}_t)$ \, $\uparrow$ & 60.35 & 58.63 & 59.81 & 58.80 & 57.61 & 61.15 & 60.35 & 60.63 & 60.11 & 58.81 & \textbf{61.48} & \underline{61.3}\\
        & $Acc(\mathcal{D}_f)$ \,  & 99.40 & 58.72 & 97.50 & 97.26 & 98.54 & 99.84 & 99.38 & 96.94 & 98.36 & 67.30 & 92.96 & 62.90\\
        & \multicolumn{1}{c|}{$|~\mathrm{Diff}~|$  \, \, $\downarrow$} & 39.05 & 0.09 & 37.69 & 38.46 & 40.93 & 38.69 & 39.03 & 36.3 & 38.25 & \underline{8.49} & 31.48 & \textbf{1.6}\\
        & Time (s) \, $\downarrow$ &\multicolumn{1}{c} \textendash & 366.09 & 139.61 & \textbf{1.895} & \underline{1.93} & 92.63 & 372.10 & 79.50 & 49.45 & 19.92 & 34.58 & 18.45\\ \midrule
    \multirow{5}{*}{\begin{tabular}[c]{@{}c@{}}CIFAR-100\\ (ViT)\end{tabular}} 
        & $Acc(\mathcal{D}_r)$ \, $\uparrow$ & 98.03 & 97.33 & \underline{99.45} & 89.76 & 97.87 & 97.84 & 97.80 & 96.74 & 99.24 & 97.53 & \textbf{99.79} & 99.04\\
        & $Acc(\mathcal{D}_t)$ \, $\uparrow$ & 88.96 & 90.20 & \underline{90.40} & 80.95 & 88.65 & 89.03 & 88.84 & 88.06 & 90.18 & 89.29 & \textbf{91.02} & 85.06\\
        & $Acc(\mathcal{D}_f)$ \, & 98.16 & 90.78 & 98.62 & 88.80 & 98.04 & 97.82 & 97.98 & 96.12 & 98.40 & 85.82 & 99.20 & 87.16\\
        & \multicolumn{1}{c|}{$|~\mathrm{Diff}~|$  \, \, $\downarrow$} & 9.20 & 0.58 & 8.22 & 7.85 & 9.39 & 8.79 & 9.14 & 8.06 & 8.22 & \underline{3.47} & 8.18 & \textbf{2.10}\\
        & Time (s) \, $\downarrow$ &\multicolumn{1}{c} \textendash & 420.53 & 334.91 & \textbf{23.65} & \underline{25.10} & 189.12 & 3287.35 & 125.43 & 140.18 & 397.08 & 739.71 & 161.45\\ \midrule
    \multirow{5}{*}{\begin{tabular}[c]{@{}c@{}}Tiny ImageNet\\ (ResNet50)\end{tabular}} 
        & $Acc(\mathcal{D}_r)$ \, $\uparrow$ & 99.39 & 99.94 & 99.74 & 98.79 & 83.81 & \textbf{99.96} & 99.39 & 99.56 & 98.98 & 98.31 & \underline{99.81} & 98.45 \\
        & $Acc(\mathcal{D}_t)$ \, $\uparrow$ & 78.74 & 78.34 & 77.51 & 78.14 & 66.67 & 78.57 & \underline{78.75} & 77.69 & 77.54 & 71.75 & \textbf{78.90} & 78.65  \\
        & $Acc(\mathcal{D}_f)$ \, & 99.36 & 78.39 & 97.90 & 97.74 & 83.48 & 99.38 & 99.35 & 98.74 & 97.30 & 78.84 & 99.11 & 77.96 \\
        & $|~\mathrm{Diff}~|$ \, $\downarrow$ & 20.62 & 0.05  & 20.39 & 19.60 & 16.81 & 20.81 & 20.60 & 21.05 & 19.76 & \underline{7.09} & 20.21 & \textbf{0.69} \\
        & Time (s) \, $\downarrow$ & \multicolumn{1}{c}{\textendash} & 1027.23 & 354.36 & \underline{12.52
        }& \textbf{12.32} & 396.83 & 52.84 & 83.00 & 93.60 & 482.42 & 988.39 & 148.6 \\ \bottomrule
    \end{tabular}
    }
\end{table*}

    
    

For random-subset forgetting, 
Table~\ref{tab:compare_random_unlearn} reports the utility and efficiency of different unlearning methods. 
Figure~\ref{fig:Facc_vs_Speedup}(a)-(b) shows Speedup versus $F_{acc}$ for ResNet-18 on CIFAR-10 and VGG-16 on CIFAR-100, with additional results in Appendix~\ref{app:D}.
We evaluate privacy protection by adopting the U-LiRA~\citep{DBLP:conf/satml/HayesSTKP25} MIA based on the log-likelihood ratio,
which is considered the SOTA per-sample attack benchmark for unlearning scenarios to date. 
The detailed attack procedure is provided in Appendix~\ref{app:C}. 
Figure~\ref{fig:Facc_vs_Speedup}(c) shows MIA results and AUC score for different unlearning methods.

\method achieves the best overall performance.
We desire to achieve \(|\mathrm{Diff}|\) as small as possible. Results show that \method consistently produces the lowest \(|\mathrm{Diff}|\) across all architectures and datasets, significantly outperforming existing baselines.
A lower \(|\mathrm{Diff}|\) indicates that the model's predictions on the forget set more closely match those on an unseen test set. 
Ideally, a perfectly unlearned model behaves as if it never encountered the forget samples during training. 
Our experiments  consistently confirm that \method approaches this ideal unlearning state more closely than all other evaluated methods.
We expect \( Acc(\mathcal{D}_r) \) and \( Acc(\mathcal{D}_t) \) to remain high after unlearning. 
\method satisfies this expectation across different architectures. 
On CIFAR-10 with ResNet-18, \method achieves higher retain and test accuracy than all baselines. 
Although \method's \( Acc(\mathcal{D}_r) \) on CIFAR-100 with VGG-16 is slightly lower than than SCRUB and EU-k, \( Acc(\mathcal{D}_t) \) is by no means inferior.
This indicates that \method does not simply memorize retain samples to achieve high retain accuracy, but instead maintains retain performance while exhibiting stronger generalization. 
SCRUB also performs well in this regard, 
but its high \( Acc(\mathcal{D}_f) \) indicates incomplete forget. 
In contrast, \method maintains high retain and test accuracy while effectively reducing prediction accuracy on the forget set, thereby achieving a better balance between selective forgetting and model utility.

\begin{figure*}[t]
    \centering
    \includegraphics[width=0.9\linewidth]{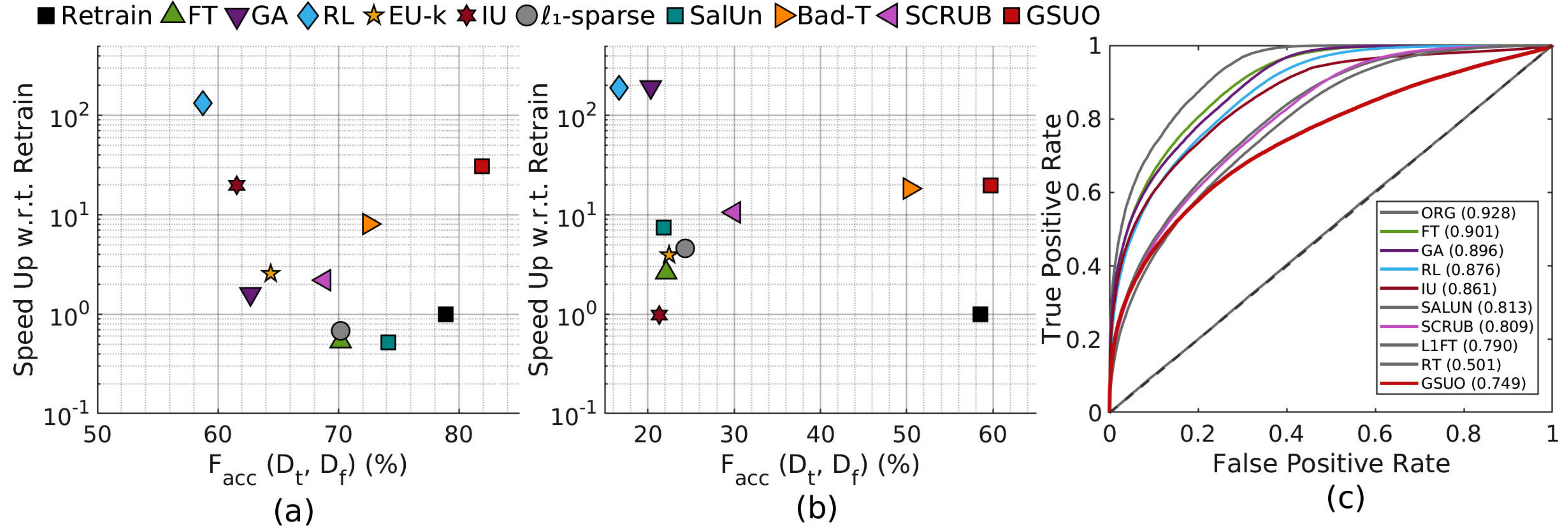}
    \caption{(a)-(b) Speedup versus $F_{acc}$ for ResNet-18 on CIFAR-10 and VGG-16 on CIFAR-100; 
        (c) MIA (U-LiRA) results and AUC score for different unlearning methods on ResNet-18 on CIFAR-10.}
    \label{fig:Facc_vs_Speedup}
\end{figure*}


For comprehensive evaluation, we adopt the metric \( F_{acc} \) to balance the model's generalization and its forgetting effectiveness on the forget set. 
We also measure computational efficiency via the speedup ratio relative to Retrain. As shown in Figure~\ref{fig:Facc_vs_Speedup}, \method consistently achieves a favorable trade-off between unlearning quality and computational efficiency (occupying the top-right region, i.e., the optimal zone with high $F_{acc}$ and high speedup). 
Although methods such as GA and RL are faster on some datasets according to Table~\ref{tab:compare_random_unlearn}, their forgetting effectiveness is far inferior to \method. This demonstrates that \method effectively balances efficiency and forgetting quality, delivering strong and robust unlearning performance without sacrificing computational practicality across diverse settings.

    

We further consider a privacy-critical application: deleting data of users who exercise their ``right to be forgotten''. To this end, we adopt U-LiRA as the privacy evaluation metric, aiming to ensure that after unlearning, an attacker should be unable to distinguish between forgotten samples and truly unseen samples, thereby protecting the privacy of users who request deletion. 
Figure~\ref{fig:Facc_vs_Speedup} shows the U-LiRA results, where \method achieves the best performance, substantially outperforming all baselines. 
This result demonstrates that \method effectively mitigates membership inference attacks on forgotten samples, offering strong privacy guarantees while maintaining competitive model utility.


\method demonstrates strong scalability and robustness as data and model scale increase. 
As forgetting tasks scale from smaller datasets with fewer classes to larger, more complex ones with more categories and samples, the forgetting exactness \(|\mathrm{Diff}|\) of \method remains consistently low, substantially outperforming all baselines. 
At the same time, it maintains competitive computational efficiency. These results verify the stability and effectiveness of \method in complex, large-scale forgetting tasks.


\subsection{Class-wise Forgetting}
\begin{table*}[t]

    \centering
    \caption{Comparison of MU methods for class-wise forgetting on Lacuna-10 under VGG-11. }
    \label{tab:compare_class_unlearn}
    
    \resizebox{0.8\linewidth}{!}{%
    \begin{tabular}{lrrrrr}
    \toprule
    \textbf{Method} & \textbf{$Acc(\mathcal{D}_r)$} \, $\uparrow$ & \textbf{$Acc(\mathcal{D}_{rt}$)} \, $\uparrow$ & \textbf{$Acc(\mathcal{D}_{f})$} \, $\downarrow$ & \textbf{$Acc(\mathcal{D}_{ft})$} \, $\downarrow$ & \textbf{Time (s)} \, $\downarrow$\\
    \midrule
    Original   & 100.00 & 87.67 & 99.74& 90.00& --- \\
    Retrain    & 99.94 & 82.67 & 0.00  & 0.00  & 396.76\\
    \midrule
    FT         & 98.43 & 86.78 & \textbf{0.00}  & \textbf{0.00}  & 50.98\\
    GA         & 88.09 & 76.67 & 0.80 & 1.00 & 49.03 \\
    RL         & 97.03 & 87.78 & 1.21 & 1.00 & 60.69 \\
    EU-k       & 98.95 & 88.22 & \textbf{0.00}  & \textbf{0.00}  & 42.37 \\
    IU         & 89.11 & 79.67 & 12.68  & 5.00  & \textbf{1.98} \\
    $\ell_1$-sparse & 82.78 & 76.00 & \textbf{0.00} & \textbf{0.00} & 51.85 \\
    SalUn      & 98.72 & 87.56 & \textbf{0.00} & \textbf{0.00} & 105.67 \\
    Bad-T & 93.59 & 81.78 & 82.70 & 67.00 & 87.13 \\
    SCRUB      & \underline{98.92} & \textbf{89.78} & 0.80  & 0.01  & 90.84 \\
    B-Shrink   & 96.39 & 85.11 & \textbf{0.00}  & \textbf{0.00}  & \underline{5.96} \\
    B-Expand   & 95.57 & 85.44 & \textbf{0.00}  & \textbf{0.00}  & 10.72 \\
    UNSIR      & 83.95 & 77.44 & \textbf{0.00}  & \textbf{0.00}  & 25.52 \\
    \midrule
    \textbf{\method} & \textbf{99.21} & \underline{89.11} & \textbf{0.00} & \textbf{0.00} & 35.59 \\
    \bottomrule
    \end{tabular}%
    }
\end{table*}

For class-wise forgetting, we evaluated the effectiveness of different unlearning methods on forgetting all samples of a randomly selected class (class 4) on the Lacuna-10 dataset. Table~\ref{tab:compare_class_unlearn} reports the corresponding effectiveness and efficiency of these methods.



We desire the training and test accuracy on the forget set to be as small as possible. 
\method achieves the optimal level on both metrics, completely removing the influence of the forgotten class from the model.
This includes eliminating predictions on both the forget set and the corresponding test samples, demonstrating that \method can fully erase the target class. 
Despite achieving perfect forgetting, \method simultaneously preserves strong performance on retain data. Specifically, it achieves the highest retain set accuracy among all methods, and its performance on the retain test set is also among the best, ranking second overall. 
While several baselines (e.g., FT, EU-k, SalUn) also achieve zero accuracy on the forget set, they either  fall short of  \method on retain set accuracy or exhibit higher computational cost. Methods such as GA, RL, and SCRUB fail to achieve complete forgetting, leaving non-zero accuracy on the forget set.

\method achieves a substantial speedup compared to retraining, 
outperforming most baselines. Although a few methods are faster, they compromise either forgetting completeness or model utility. Overall, \method strikes a strong trade-off among forgetting effectiveness, model utility, and computational efficiency, making it a practical and reliable solution for selective class removal.



\subsection{Ablation Study}
\begin{table*}[t]
    \centering
    \caption{Ablation study on \method framework}
    \label{tab:ablation}
    \resizebox{0.9\linewidth}{!}{
    \begin{tabular}{lccccccc}
        \toprule
        \multirow{2}{*}{} & \multicolumn{4}{c}{Random-subset Forgetting} & \multicolumn{3}{c}{Class-wise Forgetting} \\
        \cmidrule(lr){2-5} \cmidrule(lr){6-8}
                                       & \textbf{$Acc(\mathcal{D}_{r})$} & \textbf{$Acc(\mathcal{D}_{f})$} & \textbf{$Acc(\mathcal{D}_{t})$} & \textbf{$|~\mathrm{Diff}~|$} & \textbf{$Acc(\mathcal{D}_{rt})$} & \textbf{$Acc(\mathcal{D}_{ft})$} & \textbf{$Acc(\mathcal{D}_{t})$} \\ \midrule
        Oringinal       &   97.27       &   97.38       &  80.87 & 16.51       &    87.67      &    90.00      &  87.90        \\
        w/o \signalopt               &  99.92        &    99.46      &  83.17   & 16.29     &    87.89      &      0.00    &    79.10      \\
        w/o \compopt               &    88.64      &   78.54       &      73.43  & 5.11 &   84.56       &    6.00      &    76.70      \\
        \method                  &   99.5       &   82.25       &    82.08 & 0.17     &   89.11       &    0.00      &   80.20       \\ \bottomrule
    \end{tabular}
    }
\end{table*}

To verify the necessity of \signalopt~and \compopt, we conduct ablation studies by removing each component individually. The results are shown in Table~\ref{tab:ablation}.



For random-subset forgetting, we observe that removing either \signalopt or \compopt leads to notable performance degradation. 
Specifically, without \signalopt, the model retains excessive information on the forget set; without \compopt, test accuracy drops significantly, with \(|\mathrm{Diff}|\) remaining large in both cases.
For class-wise forgetting, we observe similar trends: removing either component results in lower retain test set accuracy compared to the full \method. We conclude that \signalopt is essential for eliminating the influence of forgotten samples, while \compopt plays a critical role in preserving model generalization.
The best forgetting performance across both tasks is achieved only when both components are present, validating the necessity of both stages in \method.

\begin{table}[t]
    \centering
    \caption{Ablation study on \signalopt}
    \label{tab:ablation2}
    \resizebox{1\linewidth}{!}{%
    \begin{tabular}{clcccccc}
        \toprule
        Task & Method & $Acc(\mathcal{D}_r)$ & $Acc(\mathcal{D}_{rt})$ & $Acc(\mathcal{D}_f)$ & $Acc(\mathcal{D}_{ft})$ & $Acc(\mathcal{D}_t)$ & $|\mathrm{Diff}|$\\
        \midrule
        \multirow{4}{*}{Random-subset} & Original & 97.27 & --- & 97.38 & --- & 80.87 & 16.51\\
                                                & w/o Boundary & 99.6 & --- & 95.62 & --- & 82.59 & 13.03\\
                                                   & w/o Normal  & 99.92 & --- & 98.8 & --- & 83.73 & 15.07\\
                                                   & \method~ & 99.5 & --- & 82.25 & --- & 82.08 & 0.17\\
        \cmidrule{1-8}
        \multirow{4}{*}{Class-wise}     & Original & 100.00 & 87.67 & 99.74 & 90.00 & 87.90 & --- \\
                                                & w/o Dispersion & 98.75 & 88.22 & 0.00 & 0.00 & 88.22 & ---\\
                                                   & w/o Alignment  & 96.45 & 84.78 & 0.00 & 0.00 & 84.78 & --- \\
                                                   & \method~ & 99.21 & 89.11 & 0.00 & 0.00 & 89.11 & --- \\
        \bottomrule
    \end{tabular}%
    }
\end{table}

To further investigate the contribution of each component in \signalopt, 
we conduct ablation experiments under two forgetting tasks. 
The results are shown in Table~\ref{tab:ablation2}. 

For Random-subset forgetting, we observe that treating all forget samples uniformly (as either all Normal or all Boundary) results in limited forgetting performance, with forget set accuracy remaining high in both cases. This finding suggests that Random-subset forgetting requires a more nuanced approach: different types of forget samples need distinct target distributions to achieve effective unlearning. In contrast, the complete \method substantially reduces forget set accuracy while keeping retain set accuracy high, achieving the optimal balance we desire. 
For the class-wise forgetting task, we observe that removing either the dispersion loss or the alignment loss still allows the model to achieve zero accuracy on the forget set. However, retain set accuracy drops noticeably compared to the full \method. This indicates that even when forget set accuracy has already reached the optimal level, both the dispersion loss and the alignment loss still help preserve model performance on the retain data. 
We therefore conclude that each component in \compopt plays a distinct yet complementary role, and only their full combination achieves the optimal trade-off between forgetting effectiveness and model utility preservation across both forgetting tasks.

\subsection{Sensitivity Analysis on Forget Set Size}

Table~\ref{tab:forgetting_ratios} shows that \method achieves a favorable trade-off between utility, forgetting efficacy, and efficiency. 
As the forgetting ratio increases from 0\% to 40\%, \method maintains high accuracy on the retain set \(\mathcal{D}_r\) , with only a 2\% drop at 40\% forgetting. 
This slight drop is reasonable, as even retraining from scratch inevitably incurs a certain degree of accuracy degradation when the size of the retain set decreases. The test set \(\mathcal{D}_t\) accuracy declines gradually from 80.80\% to 79.92\%, indicating that the model generalizes well to the overall data distribution after forgetting targeted classes. The accuracy on the forget set \(\mathcal{D}_f\) is effectively suppressed across different forgetting ratios. 
Furthermore, even at a forgetting ratio of 40\%, the \(|\mathrm{Diff}|\) metric remains below 3\%, which is even better than the performance of most existing unlearning methods at lower forgetting ratios, fully demonstrating that \method achieves efficient and stable Random-subset forgetting.
The runtime of SGA-OT increases modestly with the forgetting ratio and remains significantly lower than the cost of full retraining, confirming the practicality and scalability of our method for large-scale unlearning scenarios.

\begin{table}[t]
\centering
\caption{Performance of \method~across different forgetting ratios.}
\label{tab:forgetting_ratios}
\resizebox{1\linewidth}{!}{
\begin{tabular}{lccccccccc}
\toprule
Metrics & 0 & 5\% & 10\% & 15\% & 20\% & 25\% & 30\% & 35\% & 40\% \\
\midrule
$Acc(\mathcal{D}_r)$ \, $\uparrow$ & 92.90 & 99.51 & 99.50 & 99.00 & 98.79 & 98.23 & 98.01 & 98.05 & 97.65 \\
$Acc(\mathcal{D}_t)$ \, $\uparrow$ & 80.80 & 82.24 & 82.08 & 81.56 & 81.57 & 81.00 & 80.81 & 80.19 & 79.92 \\
$Acc(\mathcal{D}_f)$ \, & 98.34 & 81.44 & 82.25 & 81.57 & 82.88 & 84.78 & 81.34 & 82.83 & 82.62 \\
$|\,\mathrm{Diff}\,|$ \, $\downarrow$ & 17.54 & 0.8 & 0.17 & 0.01 & 1.31 & 3.78 & 0.53 & 2.64 & 2.70 \\
Time (s) \, $\downarrow$ &\multicolumn{1}{c}{\textendash} &6.6 &7.06 &8.81 &8.73 &11.97 &12.66 &11.81 &12.21 \\
\bottomrule
\end{tabular}}
\end{table}

\section{Conclusion}
We propose \method, a two-step unlearning framework consisting of Signal-Guided Optimization and Compensatory Optimization. 
By leveraging task-specific guidance signals and targeted utility restoration, \method provides a principled and efficient solution to the challenges of over-unlearning and under-unlearning in both random-subset and class-wise forgetting scenarios.
Extensive experiments demonstrate that \method consistently surpasses 14 state-of-the-art baselines across unlearning efficacy, generalization, and privacy robustness, while delivering substantial speedups, confirming its effectiveness for trustworthy machine unlearning.

\textbf{Limitation.}
\method outperforms baselines in privacy evaluation, but still falls short of the retraining upper bound. Experiments are limited to traditional models; its effectiveness on emerging architectures such as LLMs and VLMs remains to be validated.



\bibliographystyle{icml2026}
\bibliography{main}

\clearpage

\appendix
\section*{Appendix}
\section{Theoretical Analysis}
\label{app:A}

\paragraph{Notation.}



The original model \( \Theta_o \) is trained on the full dataset \( D \). 
We partition \( D \) into a retain set \( D_r \) and a forget set \( D_f \), 
where \( D_f = D_f^b \cup D_f^n \), with:

\begin{itemize}
    \item \textbf{\( D_f^b \) (Boundary samples)}: 
    The absolute logit value is small (\( |f_y(x_b; \Theta_o)| \leq \varepsilon \)), 
    which indicates low model confidence (high uncertainty). 
    These samples lie near the decision boundary, 
    and consequently their loss gradient norm \( \|\nabla_\theta \ell(z_b)\| \) is large.

    \item \textbf{\( D_f^n \) (Normal samples)}: 
    The absolute logit value is large (\( |f_y(x_n; \Theta_o)| \geq L \)), 
    which indicates high model confidence (low uncertainty). 
    These samples lie far from the decision boundary, 
    and consequently their loss gradient norm \( \|\nabla_\theta \ell(z_n)\| \) is small.
\end{itemize}

\noindent
\textit{Remark.} For cross-entropy loss, a small absolute logit (near the decision boundary) implies low confidence, high uncertainty, and a large gradient norm, and vice versa.

\paragraph{Definition.}
\label{def:mem_strength}
For a sample \( z = (x, y) \) under the original model \( \Theta_o \), the strength of memorization \( M(z) \) is defined as the magnitude of parameter change required after removing this sample from the training set. This concept can be quantified using the influence function. Let
\begin{equation}
H_{\Theta_o} = \frac{1}{|D|} \sum_{z \in D} \nabla^2_\Theta \ell(z; \Theta) \bigg|_{\Theta = \Theta_o}
\end{equation}
be the empirical Hessian matrix, where \( \ell \) is the cross-entropy loss. The influence function of sample \( z \) is defined as:
\begin{equation}
\mathcal{I}(z) = - H_{\Theta_o}^{-1} \nabla_\Theta \ell(z; \Theta_o).
\end{equation}
The strength of memorization is defined as the norm of the influence function:
\begin{equation}
M(z) = \|\mathcal{I}(z)\|_2 \propto \|\nabla_\Theta \ell(z; \Theta_o)\|_{H^{-1}}.
\end{equation}
As a first-order approximation, the strength of memorization correlates positively with the gradient norm:
\begin{equation}
M(z) \propto \|\nabla_\Theta \ell(z; \Theta_o)\|_2.
\end{equation}

\begin{theorem}[Memorization Strength Disparity]
\label{thm:1}
    Let \( z_b \in D_f^b \) be a boundary sample and \( z_n \in D_f^n \) be a normal sample. Under the original model \( \Theta_o \), the memorization strength of boundary samples is significantly larger than that of normal samples, i.e.,
    \[
    M(z_b) \gg M(z_n).
    \]
\end{theorem}

\begin{proof}
    From the notion, a boundary sample \( z_b \) lies near the decision boundary, having a small absolute logit, which implies low model confidence and high uncertainty. Consequently, its loss gradient norm \( \|\nabla_\Theta \ell(z_b; \Theta_o)\|_2 \) is large. In contrast, a normal sample \( z_n \) lies far from the decision boundary, having a large absolute logit, which implies high model confidence and low uncertainty. Consequently, its loss gradient norm \( \|\nabla_\Theta \ell(z_n; \Theta_o)\|_2 \) is small.
    
    By Definition \ref{def:mem_strength}, the strength of memorization \( M(z) \) is positively correlated with the loss gradient norm. Therefore,
    \begin{equation}
    M(z_b) \gg M(z_n).
    \end{equation}
\end{proof}

\begin{theorem}[Coarse-Grained Forgetting Dilemma]
\label{thm:2}
    When parameter updates are applied in a coarse-grained and undifferentiated manner across all samples, a fundamental dilemma emerges: samples with strong memorization (boundary samples) are left in an under-unlearned state, while samples with weak memorization (normal samples) are left in an over-unlearned state.
\end{theorem}

\begin{proof}
    Consider the coarse-grained forgetting update:
    \begin{equation}
    \Delta \Theta = \eta \cdot \frac{1}{|D_f|} \sum_{z \in D_f} \nabla_\Theta \ell(z; \Theta_o), \quad \eta > 0,
    \end{equation}
    where $\Delta \Theta = \Theta_u - \Theta_o$. Partition the forget set $D_f$ into boundary samples $D_f^b$ and normal samples $D_f^n$. Define the average gradients:
    \begin{equation}
    g_b = \frac{1}{|D_f^b|} \sum_{z_b \in D_f^b} \nabla_\Theta \ell(z_b), \qquad
    g_n = \frac{1}{|D_f^n|} \sum_{z_n \in D_f^n} \nabla_\Theta \ell(z_n).
    \end{equation}
    By Theorem~\ref{thm:1}, $\|g_b\| \gg \|g_n\|$, with boundary samples lying in high-curvature regions and normal samples in low-curvature (saturated) regions. Thus the overall update direction is dominated by $g_b$:
    \begin{equation}
    \Delta \Theta \approx \eta \cdot \frac{|D_f^b|}{|D_f|} g_b.
    \end{equation}
    Boundary samples reside in high-curvature regions where the loss is highly sensitive to parameter changes and gradients are steep. Ideal forgetting requires moving parameters along the \textit{individual gradient direction} of each sample by a moderate step size. However:

    \begin{itemize}
        \item \textbf{Direction bias}: The coarse-grained update uses the \textit{average gradient} $g_b$ over all boundary samples. Due to variability in gradient directions across samples, this average direction may not align precisely with the steepest forgetting direction for any individual $z_b$.
    
        \item \textbf{Constrained step size}: The global learning rate $\eta$ must be kept small to prevent a sharp rise in loss on the retain set $D_r$ (avoiding catastrophic forgetting). Although $\|g_b\|$ is large, the product $\|\Delta \Theta\| = \eta \|g_b\|$ remains limited.
    
        \item \textbf{High-curvature instability}: After a single large-step update, gradient directions in high-curvature regions change dramatically. A one-step coarse-grained update rarely lands exactly in the ``forgotten'' state, and without subsequent targeted adjustments, memory persists.
    \end{itemize}
    
    Thus, boundary samples are often in a relatively under-unlearned state.
    
    Normal samples reside in low-curvature (flat) regions where their own gradients satisfy $\nabla_\Theta \ell(z_n) \approx 0$ and the Hessian eigenvalues are extremely small \textit{at the basin center}. Expanding $\ell(z_n)$ around $\Theta_o$:
    \begin{equation}
\begin{split}
\ell(z_n; \Theta_o + \Delta \Theta)
&\approx \ell(z_n; \Theta_o)
+ \nabla_\Theta \ell(z_n)^\top \Delta \Theta \\
&\quad
+ \frac{1}{2}\Delta \Theta^\top H_n \Delta \Theta
+ \cdots.
\end{split}
\end{equation}
    At the flat basin center, the first-order term vanishes and the second-order term is also very small for sufficiently small $\Delta \Theta$. However, although the displacement $\Delta \Theta$ is a small perturbation from a global optimization perspective (due to the small learning rate $\eta$), it is \textit{large relative to the effective radius of the locally flat basin} where $z_n$ resides, because the Hessian eigenvalues are extremely small in that region. This displacement is driven entirely by the needs of boundary samples, not by the local geometry of $z_n$.
    
    Such a large displacement pushes the parameters away from the flat basin center toward regions where the curvature of the loss landscape increases sharply (the basin boundary). As $\Delta \Theta$ reaches these high-curvature regions, the second-order term $\frac{1}{2} \Delta \Theta^\top H_n \Delta \Theta$ and higher-order terms rapidly dominate, leading to a significant increase in loss:
    \begin{equation}
    \ell(z_n; \Theta_o + \Delta \Theta) \gg \ell(z_n; \Theta_o).
    \end{equation}
    
    Consequently, a normal sample originally at the bottom of a flat basin is forcibly ejected from its low-loss region by the large, boundary-sample-driven update. This change is not guided by $z_n$'s own gradient but is a form of ``interference'' from other samples, resulting in representation collapse and leaving normal samples in an over-unlearned state.
    
    Effective forgetting of boundary samples requires precise directional alignment. Preserving normal samples requires minimal parameter displacement, yet coarse-grained updates inevitably produce large displacements driven by $g_b$.

    No single $\eta$ can simultaneously satisfy these conflicting requirements. Hence, coarse-grained undifferentiated updates inevitably suffer from the dilemma: boundary samples are under-unlearned forgotten while normal samples are over-unlearned.
\end{proof}

\section{Detailed Related Work}
\label{app:B}
Machine unlearning was first proposed by \citep{DBLP:conf/sp/CaoY15}, introducing a method that stores statistical aggregates during training and updates them by subtracting the contribution of deleted data points before reconstructing the model. \citep{DBLP:conf/sp/BourtouleCCJTZL21} proposed the SISA training framework, which strategically limits the influence of data samples during the training process to accelerate the forgetting process.

For approximate unlearning, a variety of methods have been developed to efficiently erase data influence through post-training procedures. \citep{DBLP:conf/aaai/GravesNG21} records the data contained in each training batch and the corresponding parameter updates during the training phase. When deleting, they directly subtract from the final model parameters the updates produced by batches that contain the data to be forgotten. \citep{DBLP:conf/cvpr/GolatkarAS20} propose a model weight erasure method: the weights are first fine-tuned toward a state trained solely on retain data using a single Newton step, and then directional noise is added to disrupt the memory of forgotten data. They extended this framework to activation functions \citep{DBLP:conf/eccv/GolatkarAS20}. \citep{DBLP:conf/cvpr/GolatkarARPS21} propose Mixed Linear Forgetting, which decomposes a deep network into non-linear ``core weights'' that do not require forgetting and linear ``user weights'' that can be efficiently forgotten. \citep{DBLP:conf/icml/GuoGHM20} employ a Newton update removal mechanism, computing the influence vectors of the forgotten data points on the model and updating the parameters accordingly. \citep{DBLP:conf/aistats/IzzoSCZ21} propose projected residual updates, whose time complexity is linear in the dimension of the data to be deleted and independent of the dataset size. However, this method is primarily designed for linear models, which limits its scope of application.  \citep{DBLP:conf/iclr/FanLZ0W024} propose SALUN, which selectively updates a subset of weights using weight significance masks, thereby efficiently removing the influence of specific samples or classes while preserving overall model performance. \citep{DBLP:conf/aaai/ChundawatTMK23, DBLP:conf/nips/KurmanjiTHT23} employ a teacher-student framework for machine unlearning. In their approaches, the former proposes a dual-teacher mechanism: the student model mimics the original model (the qualified teacher) on the retain set, and mimics a randomly initialized model (the unqualified teacher) on the forget set. The latter, in contrast, uses only the original model as the teacher, training the student model to stay close to the teacher on the retain set while moving away from the teacher on the forget set. \citep{DBLP:journals/pami/HeLCHH25} proposed \emph{Natural Unlearning} to generate new samples by mixing each forget sample with related samples from the retain set, and then fine-tuning the original model on these mixed samples. 

Another line of work addresses the 
class-wise forgetting task, which requires the model to completely remove knowledge of an entire concept or category. 
\citep{DBLP:journals/tnn/TarunCMK24} propose an efficient machine unlearning method called UNSIR, which learns ``error-maximizing noise'' for the target class and combines a ``damage-repair'' two-step weight update to achieve thorough forgetting of single or multiple classes in a single training pass.
\citep{DBLP:conf/cvpr/ChenGL0W23,DBLP:conf/kdd/ChenG0LA025} shift the focus from parameter space to decision space. Boundary Shrink applies adversarial perturbations to assign incorrect labels to forget samples and fine-tunes the model, thereby compressing the decision boundary. In contrast, Boundary Expanding introduces a new neuron for a ``shadow class'' to disperse forget samples into other categories, actively expanding the decision boundary. For long-tailed distributions, LTMU~\citep{DBLP:conf/kdd/ChenG0LA025} generates augmented features based on the similarity of tail-class samples and reassigns labels according to distance, enabling effective unlearning of tail classes. \citep{DBLP:journals/corr/abs-2601-18650} proposes a dynamic loss reweighting method named FaLW, which mitigates issues of heterogeneous unlearning deviation and skewed unlearning deviation by estimating the unlearning deviation for each sample and introducing a class-wise balancing factor.

Beyond traditional deep learning models, machine unlearning research is gradually extending to broader modeling paradigms and diverse application scenarios. In large language models, unlearning can be used to remove the model's memory of useless or harmful knowledge \citep{DBLP:conf/emnlp/JiaZZLRDK024, DBLP:conf/acl/0010DT0024, DBLP:conf/icml/WuerkaixiWCX00S25, DBLP:conf/nips/JiLZLK0C24}.

\section{Experimental Setup}
\label{app:C}
\paragraph{Dataset.} The basic information about the dataset we used is summarized as follows.

\begin{compactitem}
    \item{\textit{CIFAR-10}~\citep{krizhevsky2009learning}:} it contains $60,000$ colored images of 
                        size $32 \times 32 \times 3$ with $10,000$ reserved for testing. There are $10$ target classes with $5,000$ training images per class.
    \item{\textit{CIFAR-100}~\citep{krizhevsky2009learning}:} it contains $60,000$ colored images with $10,000$ reserved for testing. 
                        There are $100$ target classes with $500$ training images per class.
    \item{\textit{Tiny-ImageNet}:} Tiny ImageNet~\citep{le2015tiny} 
                    contains 120,000 colored images of size $64 \times 64\times 3$ with 10,000 reserved for validation and 10,000 reserved for testing. 
                    There are 200 target classes with 500 training images per class.

    \item{\textit{Lacuna-10}:} A subset selected from VGGFace~\citep{DBLP:conf/sp/CaoY15}. We
randomly select 10 celebrities (classes) with at
least 500 samples.
\end{compactitem}

\paragraph{Baselines.} 
We include and compare against various SOTA machine unlearning methods, which are briefly summarized as follows.
\begin{compactitem}
    \item{\textbf{Original}:} the model trained on all data $\mathcal{D}$ without performing any unlearning. 
    \item{\textbf{Retrain}:} retraining the model from scratch without the forget set, i.e., only $\mathcal{D}_r$, 
            which is assumed not to be viable in practice, but included as a reference.
    \item{\textbf{Fine-tuning}:} finetune the original model on $\mathcal{D}_r$. 
    \item{\textbf{Gradient Ascent}~\citep{DBLP:conf/cvpr/GolatkarAS20}:} finetune the original model on the forget set $\mathcal{D}_f$ by negating the gradient.
    \item{\textbf{Random Labels}~\citep{DBLP:conf/aaai/GravesNG21, DBLP:journals/corr/abs-2012-11849}:} finetune the original model on the random relabeled forget data.
    \item{\textbf{Influence Unlearning}~\citep{DBLP:conf/icml/KohL17,DBLP:conf/aistats/IzzoSCZ21}:}
        utilize an influence function to estimate the updates required for the model weights as a result of 
        removing $\mathcal{D}_f$ from $\mathcal{D}$. 
    \item{\textbf{EU-k}~(\textit{Exact Unlearning-k}) and \textbf{CF-k}~(\textit{Catastrophic Forgetting-k})~\citep{DBLP:journals/corr/abs-2201-06640}:} 
        freeze the first $k$ layers of the original model and either train the remaining layers from scratch on $\mathcal{D}_r$ 
        or finetune the remain layers on $\mathcal{D}_r$.
    \item{\textbf{$\ell_1$-sparse}~\citep{DBLP:conf/nips/JiaLRYLLSL23}:} introduce $\ell_1$-norm for model weights and finetune the model over $\mathcal{D}_r$. 
    \item{\textbf{SalUn}~\citep{DBLP:conf/iclr/FanLZ0W024}:} draw a parallel with input saliency in model explanation to only update influential weights filtered by the threshold, and integrate with Random Label.
    \item{\textbf{Bad-Teaching}~\citep{DBLP:conf/aaai/ChundawatTMK23}:} encourage the student to move close to an incompetent/dumb teacher 
        (randomly initialized model or random generators) for the forget set and close to the original model for the retain set.
    \item{\textbf{SCRUB}~\citep{DBLP:conf/nips/KurmanjiTHT23}:} move the student model away from the teacher model and formulate a min-max bi-optimization problem by extending contrastive learning. 
\end{compactitem}

Besides, we include the following specific unlearning method for the class-wise forgetting task.
\begin{compactitem}
    \item{\textbf{Boundary Shrink} and \textbf{Boundary Expanding}~\citep{DBLP:conf/cvpr/ChenGL0W23}:} 
        is designed to destroy the decision boundary by using adversarial samples for the forget data, 
        and to exploit a new area in the decision space by assigning forget samples to an extra shadow class of the original model, with the replacement of the last layer.
    \item{\textbf{UNSIR}~\citep{DBLP:journals/tnn/TarunCMK24}:} learn an error-maximizing noise matrix for the class to be unlearned and uses it to impair the forget class and repair the original model.
     \item{\textbf{LTMU}~\citep{DBLP:journals/pami/HeLCHH25}:} Directionally rectifies the decision boundary of tail classes and guides their features to contract toward the feature space of neighboring classes for long-tailed unlearning.
\end{compactitem}

For the \textbf{random-subset forgetting task}, we choose to forget about $\approx 10\%$ of the training data.
Therefore, for the CIFAR-10 and CIFAR-100,
we have $|\mathcal{D}_f| = 5K$, and the sizes of train set $\mathcal{D}$, test set $\mathcal{D}_t$, and validation set $\mathcal{D}_v$ correspond to $45K$, $10K$, $5K$ respectively. For the Tiny-ImageNet dataset, we have $|\mathcal{D}_f| = 10K$. 
For the \textbf{class-wise forgetting task}, we used CIFAR-10 and CIFAR-100, and randomly chose class $4$ as the forgetting class for both datasets. Consequently, we have $|\mathcal{D}_f| = 5K$ for CIFAR-10 and $|\mathcal{D}_f| = 500$ for CIFAR-100. In addition, for Lacuna-10, a subset selected from VGGFace~\citep{DBLP:conf/sp/CaoY15} where each selected class contains at least 500 samples, we also randomly choose class $4$ as the forgetting class.

\paragraph{Model.}
For All-CNN~\citep{DBLP:journals/corr/SpringenbergDBR14}, we reduce the number of layers and introduce batch normalization before each non-linearity layer, training the model from scratch. 
For ResNet, we adopt different architectures for different unlearning tasks. For Random-subset forgetting, we utilize the standard ResNet-18 architecture, starting with the pre-trained PyTorch model (trained on ImageNet) and fine-tuning it on the target dataset. For class-wise forgetting, we employ a CIFAR-adapted ResNet-18 architecture, where the first $7 \times 7$ convolutional layer is replaced with a $3 \times 3$ convolutional layer and the initial max-pooling layer is removed to accommodate $32 \times 32$ inputs, training the model from scratch as the original model.
In addition, following \citep{DBLP:conf/nips/KurmanjiTHT23}, we scale up to a larger setting by choosing larger models and tackling a larger classification task, that is,
we used the pre-trained VGG-16~\citep{DBLP:journals/corr/abs-2012-11849} and ViT~\citep{DBLP:conf/iclr/DosovitskiyB0WZ21} models, finetuning them on the CIFAR-100 dataset.

\paragraph{LiRA setup.}
To evaluate the privacy protection effectiveness of MU methods, we implemented the U-LiRA~\citep{DBLP:conf/satml/HayesSTKP25} membership inference attack based on the log-likelihood ratio, 
which is considered the strongest per-sample attack benchmark for unlearning scenarios to date.
To evaluate the privacy protection effectiveness of MU methods, we implemented the U-LiRA~\citep{DBLP:conf/satml/HayesSTKP25} membership inference attack based on the log-likelihood ratio, 
which is considered the strongest per-sample attack benchmark for unlearning scenarios to date. 
The specific attack procedure is as follows:
\begin{compactitem}
    \item{\textbf{Distribution Construction and Sample Filtering:}} For each sample in the restricted sampling pool, we iterated over all shadow models and filtered them based on the sample's membership status in the models. Specifically, we collected the model outputs when the sample was in the forget set to construct the $f_{\mathrm{out}}$ distribution, and collected the model outputs when it was in the unseen sample set to construct the $f_{\mathrm{in}}$ distribution. Thanks to the aforementioned restricted sampling strategy, each target sample obtained a sufficient number of shadow model observations ($N = 128$) in both states, meeting the statistical requirements for precise Gaussian distribution fitting.
    
    \item{\textbf{Attack Scoring:}} For each target variant model, we used the log-likelihood ratio of the query sample under the two Gaussian distributions as the attack score.
    
    \item{\textbf{Evaluation Metrics:}}We conducted evaluations within the fixed pool, selecting samples that were neither in the forget set nor in the retain set as unseen test samples. To ensure unbiased evaluation metrics, we strictly implemented a 1:1 downsampling balancing strategy for positive and negative samples when calculating the AUC and attack accuracy.
\end{compactitem}

We conduct experiments on CIFAR-10 using the standard ResNet-18 architecture. To simulate realistic scenarios and construct distributions for membership inference attacks, following~\citep{DBLP:conf/satml/HayesSTKP25}, we randomly sample 25,000 samples from the CIFAR-10 training set as the training set for each model. 

U-LiRA partitions all base models into 128 shadow models and 128 target models, used for estimating distribution parameters and evaluating attack accuracy, respectively. To ensure statistical significance with a limited number of variants, we introduce a restricted sampling pool mechanism: we pre-sample a fixed pool of 1,000 samples from class 5 of CIFAR-10. Unlike the original U-LiRA's large-scale repeated sampling, we adopt a more efficient strategy — when constructing the forget set for each variant, we sample only 200 samples from the intersection of the pool and the base model's training set. Each base model randomly samples 5 different forget sets, generating 5 variants, resulting in 1,280 trained models in total. This strategy effectively increases the probability that samples from the fixed pool appear in the forget sets of various variants, ensuring that samples receive sufficient observations to accurately fit the distributions even when the total number of models is reduced.

\section{More Experimental Results.}
\label{app:D}
\subsection{Random-subset Forgetting.}
\begin{figure*}[t]
    \centering
    \includegraphics[width=0.8\linewidth]{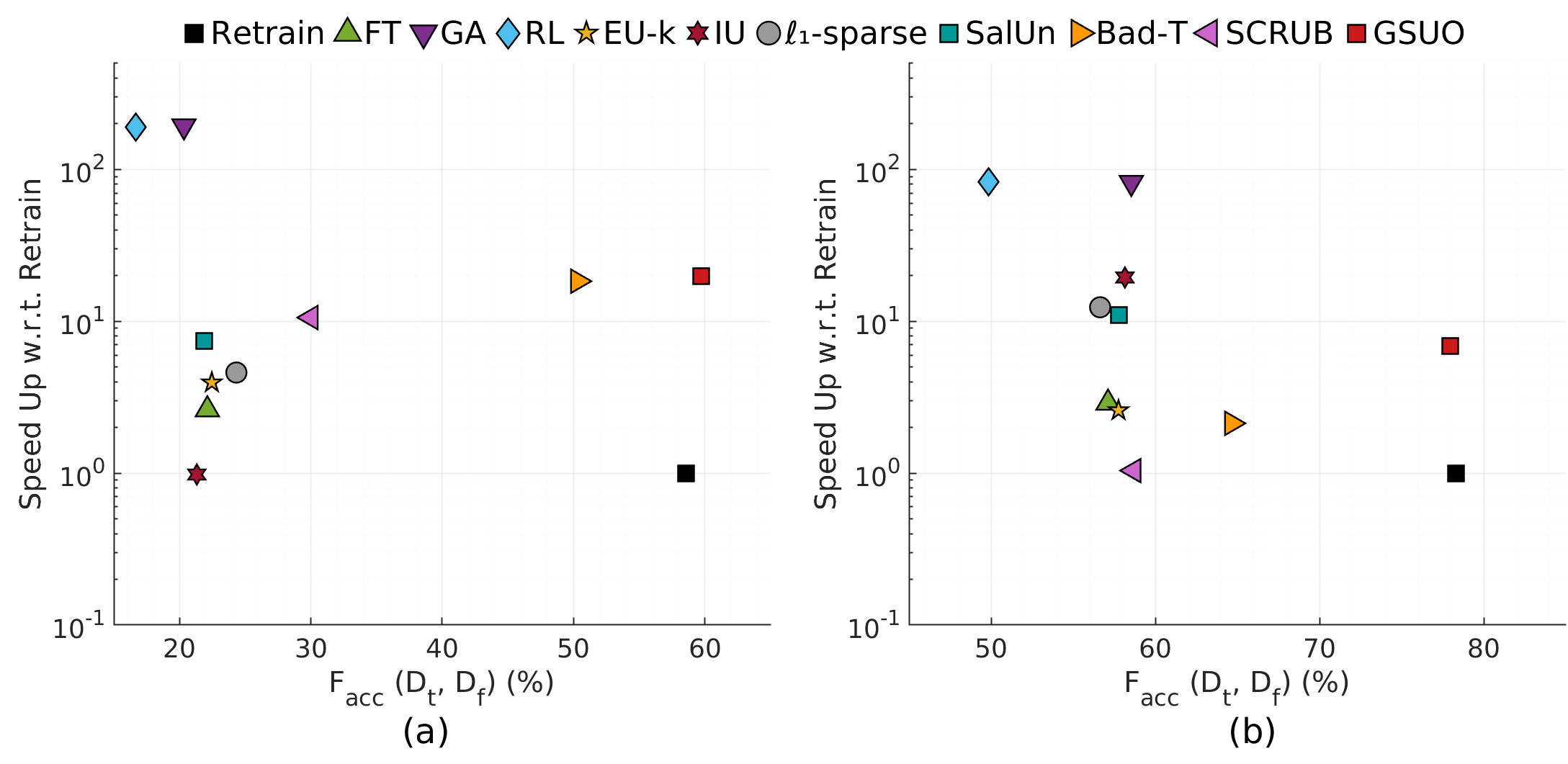}
    \caption{(a) Speedup w.r.t. Retrain versus $F_{acc}$ for vit on CIFAR-100; (b) Speedup w.r.t. Retrain versus $F_{acc}$ for ResNet50 on Tiny ImageNet.}
    \label{fig:facc2}
\end{figure*}

Figure~\ref{fig:facc2} illustrates the trade-off between speedup  and $F_{acc}$ under two settings: (a) ViT on CIFAR-100, and (b) ResNet50 on Tiny ImageNet. $F_{acc}$ jointly considers unlearning effectiveness and model generalization. 
On this metric, \method~outperforms all baselines, achieving state-of-the-art performance. 
In contrast, methods such as FT, SalUN, and IU either retain high accuracy on the forget set (indicating incomplete unlearning) or suffer from a significant drop in accuracy on the test set (indicating poor generalization). 
While efficiency-oriented methods like GA achieve shorter runtime, their low $F_{acc}$ reveal a failure to properly balance the trade-off between unlearning effectiveness and model generalization. 
In contrast, \method~consistently maintains its superiority across different datasets and model architectures, demonstrating strong scalability and robustness. 
Overall, \method~not only surpasses all competing methods in overall performance but also flexibly adapts to unlearning tasks of various scales, making it a practical and robust solution.

\begin{figure*}[t]
    \centering
    \includegraphics[width=0.8\linewidth]{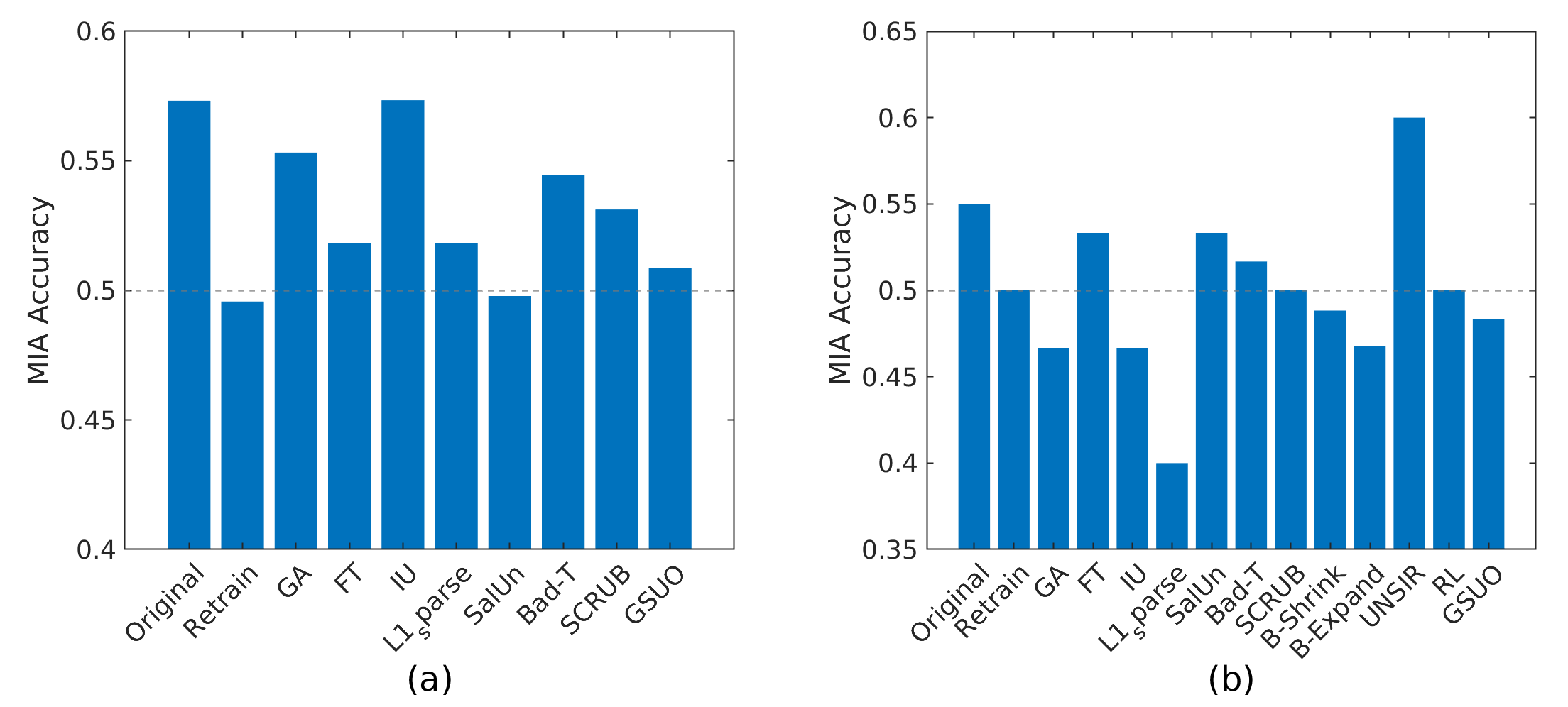}
    
    \caption{(a) MIA performance of various UM methods on the Random-subset forgetting task; (b) MIA performance of various UM methods on the class-wise forgetting task.}
    
    \label{fig:mia}
\end{figure*}

Figure~\ref{fig:mia}(a) presents the MIA attack success rate of various unlearning methods on the Random-subset forgetting task, using an SVM attack model with loss as the attack feature. The MIA accuracy of \method~is very close to 0.5, demonstrating strong protection of the privacy of the data intended to be forgotten and confirming truly effective unlearning. In contrast, methods such as IU and GA fail to achieve this level of privacy protection, with their MIA accuracy deviating significantly from 0.5.

\subsection{Class-wise Forgetting.}

\begin{table*}[t]
    \centering
    \caption{Comparison of MU methods for class-wise forgetting. Left: ALLCNN on Cifar10; Right: Resnet18 on Cifar10}
    \label{tab:compare_class_unlearn2}
    \resizebox{0.8\linewidth}{!}{%
    \begin{tabular}{lccc|cccc}
    \toprule
    \multirow{2}{*}{\textbf{Method}} & \multicolumn{3}{c|}{\textbf{ALLCNN on Cifar10}} & & \multicolumn{3}{c}{\textbf{Resnet18 on Cifar10}} \\
    \cmidrule(lr){2-4} \cmidrule(lr){6-8}
    & \textbf{$Acc(\mathcal{D}_f)$  $\downarrow$} & \textbf{$Acc(\mathcal{D}_r)$  $\uparrow$} & \textbf{Time (s) $\downarrow$} & & \textbf{$Acc(\mathcal{D}_f)$  $\downarrow$} & \textbf{$Acc(\mathcal{D}_r)$  $\uparrow$} & \textbf{Time (s) $\downarrow$} \\
    \midrule
    Original   & 83.40 & 82.48 & 189.32 & & 85.12 & 85.40 & 1038.50 \\
    Retrain    & 0.00  & 85.10 & 428.64 & & 0.00  & 84.07 & 459.81 \\
    \midrule
    GA         & 12.00 & 76.49 & \underline{5.62}   & & 9.40  & 77.50 & 145.99 \\
    B-Shrink   & 10.40 & 75.43 & 48.59  & & 1.20  & 77.08 & 309.46 \\
    B-Expand   & 10.90 & 78.34 & 9.71   & & 11.50 & 64.90 & 27.32 \\
    LTMU       & \textbf{0.00} & 84.67 & 24.16  & & \textbf{0.00} & 85.01 & 53.27 \\
    SCRUB      & 0.20  & \underline{86.02} & 146.05 & & 0.20  & \underline{86.02} & 146.05 \\
    Bad-T      & 45.60 & 78.40 & 268.24 & & 58.00 & 73.14 & 102.94 \\
    SalUn      & \textbf{0.00} & 85.47 & 50.37  & & \textbf{0.00} & 85.32 & 22.03 \\
    RL         & 1.10  & 75.93 & 203.33 & & 9.70  & 79.91 & \underline{6.74} \\
    UNSIR      & 0.30  & 80.16 & 6.13   & & 9.10  & 83.04 & 7.79 \\
    L1-sparse  & \textbf{0.00} & 77.58 & 28.93  & & \textbf{0.00} & 81.26 & 19.53 \\
    FT         & \textbf{0.00} & \underline{86.02} & 93.52  & & \textbf{0.00} & 85.04 & 18.91 \\
    IU         & 13.10 & 77.22 & \textbf{2.68}   & & 12.30 & 80.77 & \textbf{2.83} \\
    \midrule
    \method     & \textbf{0.00} & \textbf{86.92} & 47.86  & & \textbf{0.00} & \textbf{86.24} & 42.55 \\
    \bottomrule
    \end{tabular}%
    }
\end{table*}

Table~\ref{tab:compare_class_unlearn2} presents the results of class unlearning on CIFAR-10 using the ALLCNN and ResNet18 models. Our method \method~achieves 0\% accuracy on the forget set test split, while attaining the highest accuracy on the retain set test split, surpassing all baselines. 
This demonstrates that \method~effectively removes the targeted class information while well preserving the model's utility. In terms of efficiency, \method~achieves nearly a 10$\times$ speedup compared to retraining. Although it is slightly less efficient than methods such as IU and UNSIR, these baselines either suffer from high accuracy on the forget set, indicating incomplete unlearning or cause a significant drop in retain accuracy, leading to poor generalization. Overall, \method~enables efficient unlearning and shows significant advantages in both effectiveness and efficiency, making it a practical and robust choice for real-world unlearning tasks. 

Figure~\ref{fig:mia}(b) presents the MIA attack success rate of various unlearning methods on the class-wise forgetting task. The MIA accuracy of \method~is close to 0.5. Although methods such as SCRUB and RL achieve values even closer to 0.5, considering their overall unlearning effectiveness, they lag behind \method~in terms of the trade-off between privacy protection and model utility.

\subsection{Performance for imbalanced datasets.}
To investigate the performance of \method~on imbalanced data, we construct imbalanced datasets using an exponential decay strategy. 
\paragraph{Dataset}Specifically, we build a long-tailed dataset CIFAR-10-LT based on CIFAR-10, where a few head classes have abundant samples while many tail classes have only limited samples. The number of training samples per class is controlled by:
\begin{equation}
N_i = N_{\mathrm{max}} \times \left( \frac{N_{\mathrm{min}}}{N_{\mathrm{max}}} \right)^{\frac{i}{C-1}}
\end{equation}
where $C$ is the total number of classes, $N_{\mathrm{max}}$ and $N_{\mathrm{min}}$ denote the maximum and minimum sample sizes per class, and $i$ indexes the classes sorted from the most frequent to the least frequent.

\paragraph{Evaluation Metric}To address the potential performance imbalance between head classes (majority classes with many samples) and tail classes (minority classes with few samples) in long-tailed distributions, we introduce two fairness metrics to evaluate model performance across different classes.The first metric is the Weighted Accuracy Difference ($\mathrm{WAD}$), defined as follows:
\begin{equation}
\mathrm{WAD} = \sqrt{\frac{C}{\sum_{c=1}^{C} p_c \cdot (\mathrm{Acc}_c - \mathrm{Acc})^2}}
\end{equation}
where $p_c$ is the sample proportion of class $c$, $\mathrm{Acc}_c$ is the model's accuracy on class $c$, and $\mathrm{Acc}$ is the overall average accuracy. A smaller $\mathrm{WAD}$ indicates more balanced performance across classes.The second metric is the Head-Tail Performance Ratio ($\mathrm{HTR}$), defined as:
\begin{equation}
\mathrm{HTR} = \frac{\frac{1}{|G_h|} \sum_{c \in G_h} \mathrm{Acc}_c}{\frac{1}{|G_t|} \sum_{c \in G_t} \mathrm{Acc}_c}
\end{equation}
where $G_h$ and $G_t$ denote the sets of head classes and tail classes, respectively, and $|G_h|$ and $|G_t|$ are their corresponding cardinalities. $\mathrm{HTR}$ measures the ratio of average accuracy between head and tail classes, reflecting the model's fairness toward tail classes.

\begin{table*}[t]
    \centering
    \caption{Performance of Different UM Methods on CIFAR-10-LT.}
    \label{tab:compare_class_unlearn_lt}
    \resizebox{0.85\linewidth}{!}{%
    \begin{tabular}{lcccc|cccc}
    \toprule
    \multirow{2}{*}{\textbf{Method}} & \multicolumn{4}{c|}{\textbf{Head class}} & \multicolumn{4}{c}{\textbf{Tail class}} \\
    \cmidrule(lr){2-5} \cmidrule(lr){6-9}
    & \textbf{Acc($\mathcal{D}_f$)} \% $\downarrow$ & \textbf{Acc($\mathcal{D}_r$)} \% $\uparrow$ & \textbf{WAD $\downarrow$} & \textbf{HTR} & 
    \textbf{Acc($\mathcal{D}_f$)} \% $\downarrow$ & \textbf{Acc($\mathcal{D}_r$)} \% $\uparrow$ & \textbf{WAD $\downarrow$} & \textbf{HTR} \\
    \midrule
    Original   & 85.5 & 44.5 & 0.1997 & 1.9225 & 69.2 & 60.3 & 0.1761 & 1.4678 \\
    Retrain    & 0 & 64.0 & 0.1251 & 1.2605 & 0 & 73.1 & 0.1229 & 1.2069 \\
    \midrule
    GA         & 23.2 & 43.5 & 0.2250 & 1.5279 & 3.5 & 56.7 & 0.2220 & 1.9328 \\
    B-Shrink   & 0.2 & 37.7 & 0.2737 & 1.4455 & 2.2 & 55.6 & 0.2460 & 1.8436 \\
    B-Expand   & 12.0 & 42.3 & 0.1314 & \textbf{0.9793} & 2.2 & 55.9 & 0.2427 & 1.8105 \\
    LTMU       & \textbf{0} & 50.3 & 0.1940 & 1.5196 & 0.6 & 59.0 & 0.1974 & 1.7225 \\
    SCRUB      & 2.4 & \textbf{65.8} & \textbf{0.1228} & 1.2312 & 15.2 & 70.6 & \underline{0.1238} & 1.2308 \\
    Bad-T      & 77.5 & \underline{62.8} & 0.1445 & 1.2777 & 47.9 & 67.4 & 0.1963 & 1.4779 \\
    SalUn      & \textbf{0} & 48.6 & 0.1792 & 1.6420 & 5.6 & 71.0 & 0.1212 & 1.2490 \\
    RL         & \textbf{0} & 37.8 & \underline{0.1233} & 0.8997 & 1.8 & 55.6 & 0.2401 & 1.7449 \\
    UNSIR      & \textbf{0} & 42.4 & 0.2222 & 0.3938 & \textbf{0} & 61.1 & 0.1831 & \textbf{0.9182} \\
    L1-sparse  & \textbf{0} & 44.8 & 0.2607 & 1.3600 & \textbf{0} & 73.4 & 0.1242 & 1.2903 \\
    FT         & \textbf{0} & 57.6 & 0.1633 & 1.3435 & \textbf{0} & \textbf{75.8} & 0.1083 & \underline{1.1547} \\
    IU         & 23.2 & 44.4 & 0.2714 & 1.4436 & \textbf{0} & 62.2 & 0.1956 & 1.6674 \\
    \midrule
    \method     & \textbf{0} & 54.3 & 0.1428 & \underline{1.0579} & \textbf{0} & \underline{74.6} & \textbf{0.1080} & 1.1613 \\
    \bottomrule
    \end{tabular}%
    }
\end{table*}

Table~\ref{tab:compare_class_unlearn_lt} presents the performance of various unlearning methods on CIFAR-10-LT. We evaluate both head and tail classes. 
Results show that \method~achieves complete unlearning while maintaining high retain test accuracy, surpassing both the original model and methods such as B-shrink, B-Expand, and LTMU. 
On tail classes, the test accuracy of \method~even exceeds that of retrain. 

From the perspective of fairness, \method~also outperforms many baselines. For the evaluation metrics, a smaller $\mathrm{WAD}$ indicates better performance, while $\mathrm{HTR}$ values closer to $1$ are preferable. As shown in Table~\ref{tab:compare_class_unlearn_lt}, \method~ranks among the top, achieving an $\mathrm{HTR}$ of $1.0579$ on head classes and a $\mathrm{WAD}$ of $0.1080$ on tail classes. In contrast, methods such as B-shrink, IU, and SalUN fail to adequately balance performance between head and tail classes.

\subsection{Multi-Class Forgetting.}
\begin{table*}[t]
    \centering
    \caption{Comparison of MU methods for multi-class forgetting on Cifar10 under ALLCNN. }
    \label{tab:compare_multi_class_unlearn}
    
    \resizebox{0.85\linewidth}{!}{%
    \begin{tabular}{lrrrrr}
    \toprule
    \textbf{Method} & \textbf{$Acc(\mathcal{D}_r)$} \, $\uparrow$ & \textbf{$Acc(\mathcal{D}_{rt})$} \, $\uparrow$ & \textbf{$Acc(\mathcal{D}_{f})$} \, $\downarrow$ & \textbf{$Acc(\mathcal{D}_{ft})$} \, $\downarrow$ & \textbf{Time (s)} \, $\downarrow$\\
    \midrule
    Original   & 92.70 & 84.59 & 88.23 & 74.50 & 205.81 \\
    Retrain    & 99.96 & 89.89 & 0.00  & 0.00  & 401.91 \\
    \midrule
    FT         & 99.41 & 90.31 & \textbf{0.00}  & \textbf{0.00}  & 84.21 \\
    GA         & 71.73 & 67.35 & 3.39 & 3.10 & 21.02 \\
    RL         & 77.99 & 70.57 & 0.06 & 0.05 & 31.35 \\
    EU-k       & 96.54 & 90.09 & \textbf{0.00}  & \textbf{0.00}  & 398.99 \\
    IU         & 90.46 & 84.16 & 42.98  & 34.50  & \textbf{4.22} \\
    $\ell_1$-sparse & 86.04 & 82.55 & \textbf{0.00} & \textbf{0.00} & 422.61 \\
    SalUn      & \underline{99.58} & 90.30 & 0.04 & 0.05 & 499.14 \\
    Bad-T & 92.14 & 86.35 & 81.56 & 70.70 & 121.34 \\
    SCRUB      & \textbf{99.72} & \textbf{92.26} & 29.06  & 21.65  & 265.85 \\
    B-Shrink   & 70.46 & 65.86 & 4.08  & 3.75  & 29.73 \\
    B-Expand   & 90.35 & 83.50 & 28.88  & 23.10  & 228.35 \\
    UNSIR      & 87.33 & 83.29 & 0.06  & 0.05  & \underline{11.26} \\
    \midrule
    \textbf{\method} & 98.85 & \underline{90.75} & \textbf{0.00} & \textbf{0.00} & 52.64 \\
    \bottomrule
    \end{tabular}%
    }
\end{table*}

We evaluated multi-class forgetting performance of various machine unlearning methods on the CIFAR-10 dataset with the AllCNN architecture. We selected class~3 and class~4. 
The results are shown in Table~\ref{tab:compare_multi_class_unlearn}. 
while methods like FT, EU-k, and $\ell_1$-sparse achieve complete forgetting, others such as GA, IU, SalUn, Bad-T, SCRUB, and UNSIR fail to fully erase the target information. 
Notably, among all evaluated methods, \method~achieves the best balance between forgetting quality and model utility.
Furthermore, \method~achieves an 8$\times$ speedup compared to retraining. Although IU has the greatest advantage in runtime, its forgetting quality is poor. In summary, \method~achieves an optimal trade-off among forgetting completeness, model utility, and computational efficiency in multi-class forgetting tasks, validating its effectiveness in complex class-incremental scenarios.

\subsection{Visual Analysis.}
\begin{figure*}[t]
    \centering
    \includegraphics[width=0.85\linewidth]{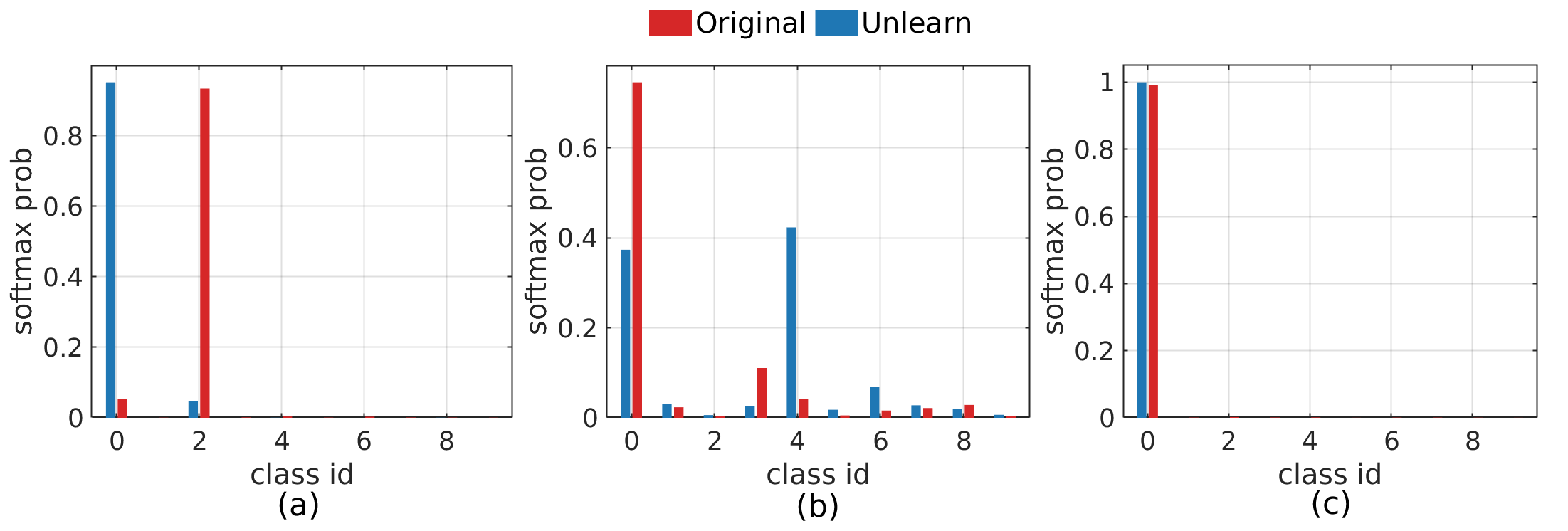}
    
    \caption{(a)-(c) Distribution of prediction changes on three forgotten samples (Random-subset forgetting).}
    
    \label{fig:distribution}
\end{figure*}

Figure~\ref{fig:distribution} illustrates the distribution of three samples before and after forgetting in Random-subset forgetting. 
Panels (a) and (b) show predicted distributions of randomly selected Boundary samples before and after forgetting, while (c) shows a randomly selected Normal sample.
Boundary samples typically not only have high probability on their true class but also exhibit a spread of probability mass over several other classes. 
After forgetting, the boundary samples undergo a probability mass transfer across the boundary toward the most similar class. 
This change encourages the boundary samples to be naturally forgotten while causing minimal damage to neighboring sample points. 
Normal samples, which are originally located in the dense region of the entire class data, tend to remain correctly predictable after forgetting, but their confidence decreases, as shown in (c). 
Visualization of predicted distributions shows that \method~successfully forgets the target samples.

\begin{figure*}[t]
    \centering
    \includegraphics[width=0.85\linewidth]{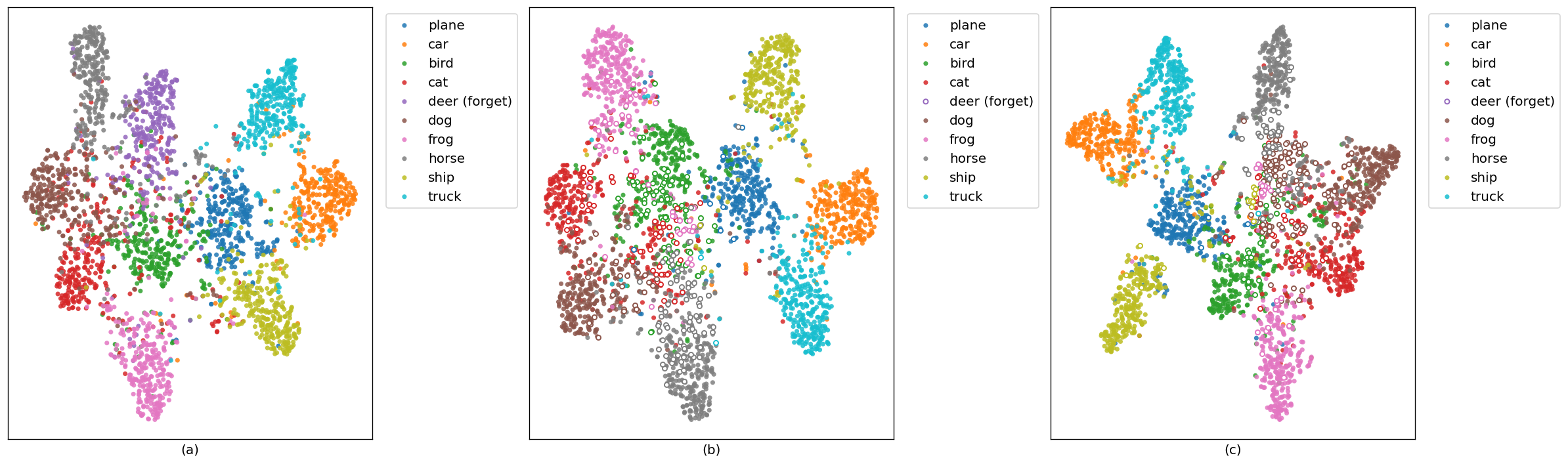}
    
    \caption{Distribution of prediction changes on forgotten samples (class-wise forgetting): (a) Original Model; (b) Retrained Model; (c) Unlearned Model.}
    
    \label{fig:distribution2}
\end{figure*}

Figure~\ref{fig:distribution2} illustrates the feature space distributions of the original model, the retrained model, and the unlearned model in class-wise forgetting. 
The experiments use CIFAR-10 with the AllCNN model. 
Features are projected onto a two-dimensional plane using t-SNE.
The unlearned model’s distribution closely resembles that of the retrained model. 
The intra-class structure of the forgotten data in the feature space tends to become loose, dispersedly distributed in the central vacant regions of the retain data. 
Most forgotten samples are predicted by the model as the class labels of the nearest retain data in the feature space. The above results demonstrate that \method~achieves the desired forgetting effect in the class-wise forgetting scenario, making its distribution in the feature space highly similar to that of the ``gold standard''——the retrained model.


\end{document}